\documentclass[journal]{IEEEtran}
\usepackage{cite}

\ifCLASSINFOpdf
\else
\fi
\usepackage{amsmath}
\usepackage{amssymb}
\usepackage{amsthm}
\DeclareMathOperator{\argmax}{arg\,max}

\newtheorem{theorem}{Theorem}
\newtheorem{lemma}{Lemma}

\newtheorem{corollary}{Corollary}
\newtheorem{definition}{Definition}

\usepackage{algorithmic}
\usepackage[lined,boxed,commentsnumbered,ruled]{algorithm2e} 
\usepackage{graphicx}
\usepackage{xcolor}

\usepackage{array}
\usepackage{makecell}
\usepackage{multirow}
\usepackage{booktabs}

\ifCLASSOPTIONcompsoc
 \usepackage[caption=false,font=normalsize,labelfont=sf,textfont=sf]{subfig}
\else
 \usepackage[caption=false,font=footnotesize]{subfig}
\fi
\usepackage{url}

\begin{document}
%
\title{Efficient Combinatorial Optimization for Word-level Adversarial Textual Attack}
%
%
%

\author{Shengcai~Liu,~\IEEEmembership{Member,~IEEE,}
        Ning~Lu,~\IEEEmembership{Student Member,~IEEE,}
        Cheng~Chen,~\IEEEmembership{Member,~IEEE,}
        Ke~Tang,~\IEEEmembership{Senior Member,~IEEE,}
\thanks{The authors are with the Research Institute of Trustworthy Autonomous Systems, Southern University of Science and Technology, Shenzhen 518055, China,
and the Guangdong Key Laboratory of Brain-Inspired Intelligent Computation, Department of Computer Science and Engineering, Southern University of Science and Technology, Shenzhen 518055, China (e-mail: liusc3@sustech.edu.cn, 11610310@mail.sustech.edu.cn, chenc3@sustech.edu.cn, tangk3@sustech.edu.cn).
Corresponding author: Cheng Chen.}}

%
%

\markboth{Journal of \LaTeX\ Class Files,~Vol.~14, No.~8, August~2015}%
{Shell \MakeLowercase{\textit{et al.}}: Bare Demo of IEEEtran.cls for IEEE Journals}
%



\maketitle

\begin{abstract}
  Over the past few years, various word-level textual attack approaches have been proposed to reveal the vulnerability of deep neural networks used in natural language processing.
  Typically, these approaches involve an important optimization step to determine which substitute to be used for each word in the original input.
  However, current research on this step is still rather limited, from the perspectives of both problem-understanding and problem-solving.
  In this paper, we address these issues by uncovering the theoretical properties of the problem and proposing an efficient local search algorithm (LS) to solve it.
  We establish the \textit{first} provable approximation guarantee on solving the problem in general cases.
  Extensive experiments involving 5 NLP tasks, 8 datasets and 26 NLP models show that LS can largely reduce the number of queries usually by an order of magnitude to achieve high attack success rates. 
  Further experiments show that the adversarial examples crafted by LS usually have higher quality, exhibit better transferability, and can bring more robustness improvement to victim models by adversarial training.
\end{abstract}

%
\IEEEpeerreviewmaketitle

\section{Introduction}
\IEEEPARstart{U}{nderstanding} the vulnerability of deep neural networks (DNNs) to \textit{adversarial examples} \cite{SzegedyZSBEGF14,GoodfellowSS15} has emerged as an important research area, due to the wide range of applications of DNNs.
Generally, adversarial examples are crafted by maliciously perturbing the original input, with the goal of fooling the target DNNs into producing undesirable behavior.
In image classification and speech recognition, extensive studies have been conducted for devising effective adversarial attacks \cite{GoodfellowSS15,CarliniMVZSSWZ16,DongLPS0HL18,TramerKPGBM18}, as well as further improving robustness and interpretability of DNNs \cite{WongK18}.

In the area of Natural Language Processing (NLP), motivated by the threats of adversarial examples in key applications such as spam filtering \cite{bhowmick2018mail} and malware detection \cite{McLaughlinRKYMS17}, there has been considerable attempt on addressing textual attacks for various NLP tasks \cite{wang2019towards}.
Unlike image attack, for textual attack it is difficult to exploit the gradient of the network with respect to input perturbation, due to the discrete nature of texts \cite{AlzantotSEHSC18}.
Moreover, in realistic settings (e.g., attacking web service such as Google Translate), the attacker usually has no access to the model weights or gradients, but can only submit a limited number of input queries and then receive the corresponding model predictions.
These facts give rise to the \textit{black-box} approaches that directly modify the original input to craft adversarial examples, without utilizing the internal information of the victim models.
Depending on where the modification is applied, these approaches can be further categorized into character-level manipulation \cite{HosseiniKZP17,BelinkovB18,GaoLSQ18}, word-level substitution \cite{AlzantotSEHSC18,RenDHC19,LiJDLW19,ZangQYLZLS20,JinJZS20,TanJKS20,LiMGXQ20,ZhaoDS18,Maheshwary2020}, and sentence-level paraphrasing \cite{IyyerWGZ18,SinghGR18}.
Among them the word-level attack has attracted the most research interest, due to its good performance on both attack efficiency and adversarial example quality \cite{PruthiDL19,Wangabs190906723}.

In general, word-level attack approaches adopt a two-step strategy to craft adversarial examples \cite{ZangQYLZLS20}.
At the first step, they construct a set of candidate substitutes for each word in the original input, thus obtaining a discrete search space.
As shown in Figure~\ref{fig:examples}, although different attack approaches have used different word substitution methods, they all have the same goal here --- making the potential adversarial examples in the search space valid, meanwhile maintaining good grammaticality and naturality.

\begin{figure}[tbp]
  \includegraphics[width=\columnwidth]{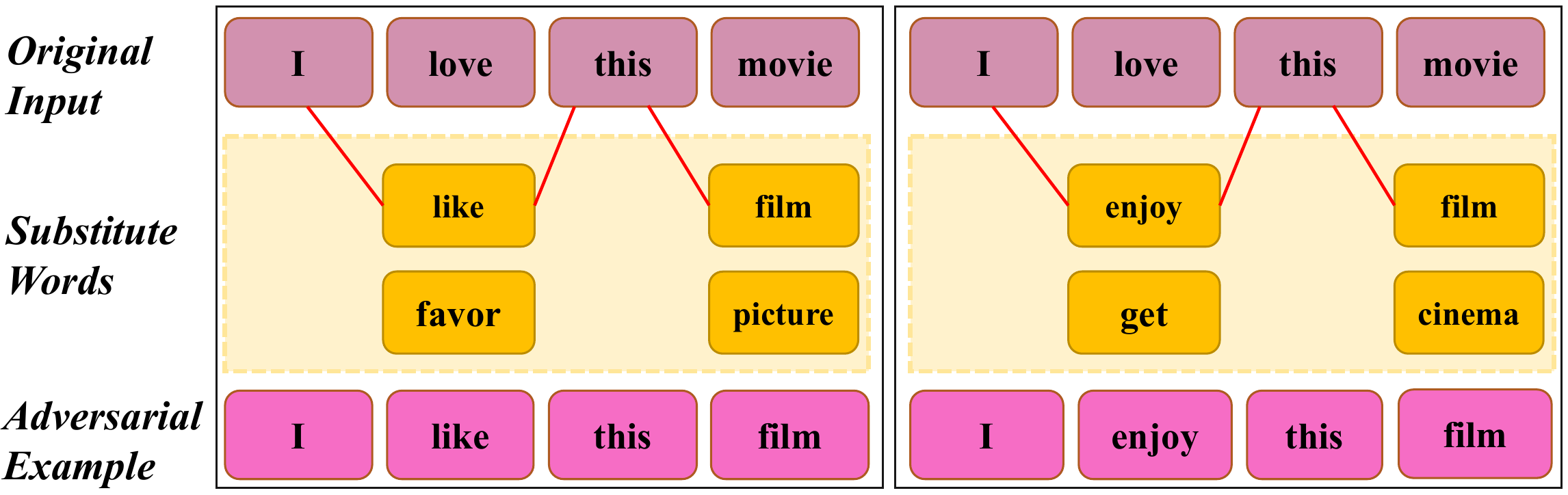}
  \caption{Simple examples showing sememe-based word substitution \cite{ZangQYLZLS20} (the left) and synonym-based word substitution \cite{RenDHC19} (the right), as well as the final crafted adversarial examples.}
  \label{fig:examples}
\end{figure}

At the second step, these approaches solve the induced \textit{combinatorial optimization} problem by determining for each word which substitute should be used, to finally craft an adversarial example.
It is conceivable that the optimization algorithm used at this step is crucial for the overall attack performance.
Previous studies have explored various options, including population-based search algorithms \cite{AlzantotSEHSC18,Wangabs190906723,ZangQYLZLS20,Maheshwary2020} and simple heuristics \cite{RenDHC19,TanJKS20,JinJZS20}.
However, there still exist some important issues that remain unsolved.
First, from the theoretical perspective, little is known about the hardness of the optimization problem considered here, nor the performance guarantee which we can achieve on it.
Second, from the practical perspective, the currently adopted optimization algorithms cannot achieve good trade-offs between attack success rates and query efficiency, considering that either low success rates or low query efficiency will make the attack approach less practical \cite{LiJDLW19,Wangabs190906723}.
Specifically, population-based algorithms need to perform excessive model queries to maintain a high attack success rate.
In contrast, simple heuristics perform much fewer queries, but they cannot achieve satisfactory success rates.
A further observation is that the optimization algorithm used in this step is actually independent of the word substitution method used in the first step.
This means a powerful optimization algorithm can be used in combination with all existing word substitution methods to achieve the best possible attack performance.
Moreover, it can also be used as a plug-and-play algorithm for the optimization module of future attack approaches.

Based on the above observations, in this work we focus on the optimization step of word-level textual attack and address the aforementioned two issues.
We first formulate the combinatorial optimization problem as the \textit{Set Maximization with Partition Matroid Constraints} which has been proven to be NP-Hard.
Then we present a simple yet powerful local search algorithm (dubbed LS) for solving it.
Intuitively, LS maximizes the marginal gain in each search step while considering multiple types of textual perturbations, which enables it to terminate quickly without performing excessive queries.
We prove an approximation bound of its performance, which is the \textit{first} performance guarantee on solving the considered problem in general cases.

We conduct large-scale experiments to evaluate LS.
Specifically, we consider five representative word-level textual attack approaches in the literature and compare LS with their original optimization algorithms.
The whole experiments involve 5 NLP tasks, 8 datasets and 26 NLP models.
The results demonstrate that LS consistently outperforms the baselines.
Compared with population-based algorithms, it dramatically reduces the number of queries usually by an order of magnitude, while achieving the same or even higher attack success rates.
Compared with simple heuristic algorithms, it obtains much higher success rates, with a reasonable number of additional queries.
Moreover, further experiments show that compared to the adversarial examples crafted by the baselines, the ones crafted by LS typically have higher quality and comparable validity,
exhibit better transferability, and can bring larger robustness improvement to victim models by adversarial training.

The rest of the paper is organized as follows.
Section~\ref{sec:related_work} presents a literature review on adversarial textual attack.
Section~\ref{sec:methods} first formally defines the induced combinatorial optimization problem and investigates its main characteristics, and then presents the local search algorithm, followed by its approximation bound.
Section~\ref{sec:exp} compares the proposed algorithm with the optimization algorithms of the state-of-the-art textual attack approaches in various attack scenarios.
Finally, Section~\ref{sec:conclusion} concludes the paper and discusses the potential future directions. 

\section{Related Work}
\label{sec:related_work}
As aforementioned, existing adversarial textual attacks can be roughly classified into sentence-level, word-level, and character-level attacks, depending on where the perturbation of the original input is applied.
Sentence-level attacks can be crafted through paraphrasing \cite{IyyerWGZ18,SinghGR18}, performing perturbations in the continuous latent semantic space \cite{ZhaoDS18}, and adding distracting sentences \cite{JiaL17}.
However, adversarial examples crafted by these approaches usually have significantly different forms from the original input and therefore it is difficult to maintain their validity.
Character-level attacks are often crafted by random character manipulation such as swap, substitution, deletion, insertion and repeating \cite{BelinkovB18,GaoLSQ18,HosseiniKZP17,LiJDLW19}.
In addition, there have also been attempts on exploiting the gradients of the victim model to guide the process of character substitution, with the help of one-hot character embeddings \cite{EbrahimiRLD18}.
Although character-level attacks can achieve high success rates, they often break the grammaticality and naturality of original input and there has been evidence in the literature that character-level attacks can be easily defended \cite{PruthiDL19}.

Compared to character-level and sentence-level attacks which tend to break either grammaticality or validity of the original input, word-level attacks usually perform well on both attack efficiency and adversarial example quality.
As aforementioned, generally a word-level attack approach can be decomposed into two steps: word substitution and optimization.
For the former, many approaches have used synonyms \cite{RenDHC19,Wangabs190906723,JinJZS20,Maheshwary2020}, infections \cite{TanJKS20} or the words with the same sememes \cite{ZangQYLZLS20}, as the candidate substitutes.
Another popular choice is to leverage language models \cite{LiMGXQ20} or word embeddings \cite{LiJDLW19,ChengYCZH20} to filter out inappropriate words.
The combination of the above two has also been explored \cite{AlzantotSEHSC18}.

As for the optimization step, existing studies have adopted population-based search algorithms including genetic algorithm \cite{AlzantotSEHSC18,Wangabs190906723,Maheshwary2020} and particle swarm optimization \cite{ZangQYLZLS20}.
Other approaches mainly use simple heuristics such as saliency/importance-based greedy substitution \cite{RenDHC19,LiJDLW19,JinJZS20,LiMGXQ20} which first sorts words according to their saliency/importance and then finds the best substitute for each word in turn,
and sequential greedy substitution \cite{TanJKS20} which sequentially finds the best substitute for each word in the original input.
However, for the optimization problem considered here, its theoretical properties are still unclear.
In addition, there is still much room for improvement regarding query efficiency and attack success rates.
It is worth mentioning there has been some well-designed toolboxes such as TextAttack \cite{MorrisLYGJQ20} that provides interfaces to easily construct textual attacks from combinations of novel and existing components.

Finally, textual attacks can also be categorized according to the accessibility to the victim model, i.e., \textit{white-box} and \textit{black-box} attacks.
White-box attacks require full knowledge of the victim model to perform gradient computation \cite{PapernotMSH16,SatoSS018,BinLiang2018,EbrahimiRLD18,WallaceFKGS19}, which however is often unavailable in practice.
In contrast, black-box attacks \cite{AlzantotSEHSC18,ZangQYLZLS20,LiJDLW19,RenDHC19,TanJKS20,JinJZS20,ZhangZML19,EgerSRL0MSSG19,IyyerWGZ18,ZhaoDS18,SinghGR18,Maheshwary2020,LiMGXQ20} only require the output of the victim model.
The present work falls into this category and assumes the class probabilities or confidence scores of the victim model are available.

\section{Methods}
\label{sec:methods}

In the optimization step of word-level textual attack, we are given an original input $\mathbf{x}=w_1...w_i...w_n$, where $n$ is the word number and $w_i$ is the $i$-th word.
Let $B_i$ denote the set of candidate substitutes for $w_i$ (note $B_i$ can be empty).
To craft an adversarial example $\mathbf{x}_{adv}$, for each $w_i$ in $\mathbf{x}$, we can select at most one word from $B_i$ to replace it (if none is selected, $w_i$ remains unchanged).
For example, in the left part of Figure~\ref{fig:examples}, we select ``like'' from $B_2$ and ``film'' from $B_4$ to replace $w_2$ (``love'') and $w_4$ (``movie''), respectively.
The goal is then to find $\mathbf{x}_{adv}$ that maximizes the objective function $f(\mathbf{x}_{adv})$, with as few model queries as possible.
In the literature, the commonly considered $f(\mathbf{x}_{adv})$ is the predicted probability on a specific wrong class \cite{ZangQYLZLS20}, i.e., targeted attack, or one minus the predicted probability on the ground truth \cite{RenDHC19}, i.e., untargeted attack.

\subsection{Problem Formulation}
\label{sec:problem_formulation}
Formally, let $[n] \triangleq \{1,2,...,n\}$ and $V \triangleq \bigcup_{i \in [n]}B_i$.
We first notice that crafting $\mathbf{x}_{adv}$ is equivalent to selecting a subset $S \subseteq V$ s.t. $|S \cap B_i| \leq 1, \forall i \in [n]$ \footnote{This requires that $B_1,...,B_n$ are disjoint, which can be guaranteed by adding a position-aware prefix to each word in $B_1,...,B_n$, For example, if $B_2$ and $B_4$ have a common element (word) ``like'', then the one in $B_2$ is changed to ``2\_like'' while the other in $B_4$ is changed to ``4\_like'', such that they are distinguished.}.
For example, the adversarial example in the left part of Figure~\ref{fig:examples} corresponds to the subset $\{\text{``like''}, \text{``film''}\}$.
Henceforth, we use $\mathbf{x}_{adv}$ and its corresponding $S$ interchangeably.
Let $d_i$ be integers s.t. $1 \leq d_i \leq |B_i|, \forall i \in [n]$.
Equivalently, the considered problem is actually a special case of the following problem with $d_i=1, \forall i \in [n]$:
\begin{equation}
\label{eq:problem_definition}
  \max_{S \subseteq V}  f(S) \ \ \mathrm{s.t.}\ |S \cap B_i| \leq d_i, \forall i \in [n].
\end{equation}
In the literature, the problem in Eq.~(\ref{eq:problem_definition}) is dubbed the \textit{Set Maximization with Partition Matroid Constraints}, which has been proven to be NP-hard \cite{cornuejols1977exceptional} in general cases.
Using the naive exhaustive search to find the exact solution to it needs $\prod_{i \in [n]} \sum_{j=0}^{d_i} {|B_i|\choose j} = \mathcal{O}(2^{|V|})$ queries.
However, classical results \cite{NemhauserWF78} have shown that the conventional greedy algorithm can achieve $(1-1/e)$-approximation if $f$ is monotone submodular (see the definitions below).
Moreover, if $f$ is non-monotone submodular, the greedy algorithm achieves $\frac{1}{\alpha}(1-e^{-\alpha \bar{d}/d})$-approximation \cite{Friedrich2019}, where $d=\sum_{i\in [n]} d_i$, $\bar{d}=\max_{i\in [n]}d_i$ and $\alpha \geq 0$ is a parameter bounding the maximum rate with which $f$ changes.
 \begin{definition}
   \label{def:monotone}
   A set function $f:2^V \rightarrow \mathbb{R}$ is monotone if for any $X \subseteq Y \subseteq V$, $f(X) \leq f(Y)$.
 \end{definition}
\begin{definition}
  \label{def:submodular}
  A set function $f:2^V \rightarrow \mathbb{R}$ is submodular if for any $X \subseteq Y \subseteq V$ and any $e \in V \setminus Y $, $f(X \cup \{e\}) - f(X) \geq f(Y \cup \{e\}) - f(Y)$.
\end{definition}
Intuitively, submodular functions exhibit a \textit{diminishing returns} property that the marginal gain of adding an element diminishes as the set size increases.
Unfortunately, in the case of textual attack, $f$ can be non-submodular.
For example, in the right part of Figure~\ref{fig:examples}, suppose we consider targeted attack for sentiment analysis and $f$ is the predicted probability on the ``negative'' label.
Using BERT \cite{DevlinCLT19} trained on IMDB data \cite{MaasDPHNP11} as the victim model, for $X=\varnothing$, $Y=\{\text{``cinema''}\}$ and $e=\text{``get''}$, we obtain $f(X)=2.432\text{E-}{5}$, $f(X\cup \{e\})=2.724\text{E-}{5}$, $f(Y)=2.664\text{E-}{5}$ and $f(Y \cup \{e\})=3.141\text{E-}{5}$.
Then $f(X\cup \{e\}) - f(X) = 2.92\text{E-}{6} < f(Y\cup \{e\}) - f(Y)=4.77\text{E-}{6}$, which contradicts with Definition~\ref{def:submodular}.
On the other hand, even if $f$ is not strictly submodular, we can still maximize $f$ to a substantial extent by finding a local optimal solution to it.
Actually, by characterizing how close $f$ is to submodularity, we can strictly bound the gap between the local optimal solution and the optimum of $f$.
That is to say, no matter whether $f$ is submodular or not, our proposed local search algorithm always achieves a provable performance guarantee when maximizing it.
In below, we first present the local search algorithm, and then establish its performance guarantee.

\begin{algorithm}[tbp]
	\LinesNumbered
	\SetKwInOut{Input}{input}
	\SetKwInOut{Output}{inout}
	\Input{objective function $f:2^V \rightarrow \mathbb{R}$; disjoint subsets $B_1,...,B_n \subseteq V$; integers $d_1,...,d_n$ s.t. $1 \leq d_i \leq |B_i|$, $\forall i \in [n]$;}
	$S \leftarrow \varnothing$;\\
	\While{\textbf{true}}
	{
		\tcc{-----insertion-----}
		Let $e \in V \setminus S$ maximizing $f(S \cup \{e\})-f(S)$ s.t. $|(S\cup\{e\}) \cap B_i| \leq d_i, \forall i \in [n]$;\\
		$S_1 \leftarrow S \cup \{e\}$;\\
		\tcc{-----deletion------}
		Let $e \in S$ maximizing $f(S \setminus \{e\})-f(S)$;\\
		$S_2 \leftarrow S \setminus \{e\}$;\\
		\tcc{-----exchange------}
		Let $e \in S$ and $v \in V \setminus S$ maximizing $f(S \setminus \{e\} \cup \{v\}) - f(S)$ s.t. $|(S \setminus \{e\} \cup \{v\}) \cap B_i| \leq d_i, \forall i \in [n]$;\\
		$S_3 \leftarrow S \setminus \{e\} \cup \{v\}$;\\
		\tcc{----update $S$ ----}
		\uIf{$\max_{A \in \{S_1,S_2,S_3\}}f(A) > f(S)$}
		{
			$S \leftarrow \argmax_{A \in \{S_1,S_2,S_3\}}f(A)$;\\
		}
		\lElse{break}
	}
	\uIf{$|(V \setminus S) \cap B_i| \leq d_i, \forall i \in [n]$}
	{
		\Return{$\argmax_{A \in \{S,V \setminus S\}}f(A)$};\\
	}
	\lElse{\Return{$S$}}
	\caption{Local Search Algorithm}
	\label{alg:local_search}
\end{algorithm}

\subsection{The Local Search Algorithm}
\label{sec:alg}

The proposed local search algorithm (LS) is outlined in Algorithm~\ref{alg:local_search}.
In brief, LS is an iterative procedure that stats from an empty set (line 1), and at each step (lines 2-16) considers three types of one-item perturbations to the selected subset $S$ and chooses the one that improves the objective at most (lines 12-14).
Specifically, the considered perturbations include:
1) \textbf{inserting} one item from $V\setminus S$ to $S$ (lines 3-5);
2) \textbf{deleting} one item from $S$ (lines 6-8);
3) \textbf{exchanging} one item from $S$ with another item from $V \setminus S$ (lines 9-11).
Note that only valid perturbations that result in feasible $S$ are considered.
The algorithm terminates when no possible improvement can be found (line 15), and returns the better one between $S$ and its complement $V\setminus S$ (if it is feasible, see lines 17-19).

When using LS to solve the optimization problem in word-level textual attack, the considered three types of one-item perturbations to $S$ correspond exactly to three types of one-word changes to the adversarial example $\mathbf{x}_{adv}$.
To illustrate this, we take the left part of Figure 1 as an example, where the original input $\mathbf{x}=\text{``I love this movie''}$ and $V=\{\text{``like'',``favor'',``film'',``picture''}\}$.
Suppose the current selected subset $S=\{\text{``like''}\}$; therefore the corresponding adversarial example $\mathbf{x}_{adv}=\text{``I like this movie''}$.
Then the three types of one-item perturbations to $S$ and the corresponding one-word changes to $\mathbf{x}_{adv}$ are as follows:
\begin{enumerate}
	\item \textbf{insertion}: inserting ``film'' or ``picture'' from $V \setminus S$ into $S$, which corresponds to substituting the word ``movie'' in $\mathbf{x}_{adv}$ with ``film'' or ``picture'';
	\item \textbf{deletion}: deleting ``like'' from $S$, which corresponds to substituting the word ``like'' in $\mathbf{x}_{adv}$ with the original word ``love'';
	\item \textbf{exchange}: exchanging ``like'' in $S$ with ``favor'' from $V \setminus S$, which corresponds to substituting the word ``like'' in $\mathbf{x}_{adv}$  with ``favor''.
\end{enumerate}
All these one-word changes to $\mathbf{x}_{adv}$ will be evaluated and the best one among them is selected to update $\mathbf{x}_{adv}$ (lines 13-15).

The total number of queries consumed by LS is also bounded.
First, LS will perform at most $\mathcal{O}(|V|)$ queries at one step (lines 3-11).
Besides, we can slightly modify the inequality condition in line 13 by adding a small positive constant $\epsilon$ to its right-hand side, to ensure that at every step the objective improvement is no less than $\epsilon$.
Since the value range of the objective function $f$ considered in textual attack is finite (typically very small), e.g., if $f$ is the predicted probability then its value range is 1, the algorithm will perform at most $\mathcal{O}(1 / \epsilon)$ steps.
This means the total query number is $\mathcal{O}(|V| / \epsilon)$.

It is worth mentioning that if only the insertion perturbation is used, then LS degenerates to the conventional greedy algorithm \cite{NemhauserWF78}.
We have also tested the greedy algorithm in the experiments (see Appendix~B), and find it consistently performs worse than LS in terms of attack success rates, which indicates the necessity of considering deletion and exchange perturbations.
Below we establish the \textit{first} performance guarantee on solving the problem in Eq.~(\ref{eq:problem_definition}) in general cases.

\subsection{Approximation Bound}
In a nutshell, we use the submodularity index \cite{ZhouS16} to characterize how close $f$ is to
submodularity, and establish an approximation bound that holds for any local optimum (such as
the solution $S$ found by LS) on $f$.
Note we assume $f$ is non-negative, which is true in word-level textual attack, and we set $f(S)$ = 0 for any infeasible solution $S$.
We first give the formal definition of local optimum.

\begin{definition}
  \label{def:local_optimum}
  Given a set function $f:2^V \rightarrow \mathbb{R}$, $S$ is a local optimum, if
  $f(S) \geq f(S\cup \{e\})$ for any $e \in V\setminus S$,
  $f(S) \geq f(S\setminus \{e\})$ for any $e \in S$,
  and $f(S) \geq f(S \setminus \{e\} \cup \{v\})$ for any $e \in S$ and $v \in V \setminus S$.
\end{definition}

The three conditions (insertion, deletion and exchange) for being a local optimum in Definition~\ref{def:local_optimum} correspond exactly to the three types of perturbations in LS.
We then have the following lemma.

\begin{lemma}
  \label{lem:local_optimum}
  Let $S$ be the solution obtained by running LS, then $S$ is a local optimum on $f$ defined in Eq.~(\ref{eq:problem_definition}).
\end{lemma}
\begin{proof}
	Suppose $S$ is not a local optimum, then there exists an element $e$ satisfying one of the following three conditions:
	1) $e \in V \setminus S$ and $f(S) < f(S\cup \{e\})$;
	2) $e \in S$ and $f(S) < f(S\setminus \{e\})$;
	3) $e \in S$ and there exists $v \in V \setminus S$, $f(S) < f(S \setminus \{e\} \cup \{v\})$.
	This implies LS must not terminate with $S$.
	Contradiction.
\end{proof}

Note that Lemma~\ref{lem:local_optimum} holds regardless of $f$ being submodular.
The following theorem presents an approximation bound achieved by any local optimum on $f$ in Eq.~(\ref{eq:problem_definition}), which also holds for the solution found by LS due to Lemma~\ref{lem:local_optimum}.

\begin{theorem}
  \label{the:bound}
  Let $C$ be an optimal solution for the problem defined in Eq.~(\ref{eq:problem_definition}) and $S$ be a local optimum.
  Then we have:
  \begin{align*}
  	2f(S)+f(V\setminus S) & \geq  f(C) + \max \{ \xi \lambda_{f}(V,2),\\ &\delta \lambda_{f}(V,2)+ \lambda_{f}(S,|C \setminus S|) \},
  \end{align*}
  where $\lambda_{f}(\cdot,\cdot)$ is the submodularity index \cite{ZhouS16},
  $\xi={ |S \setminus C| \choose 2 } + {|C \setminus S| \choose 2} + |\overline{S \cup C}| \cdot |S| + |C \setminus S| \cdot |S \cap C|$
  and
  $\delta=\xi- {|C \setminus S| \choose 2} + |C \setminus S|$.
\end{theorem}

The proof is given in the next section.
Since the previous bound \cite{Friedrich2019} for the problem defined in Eq.~(\ref{eq:problem_definition}) only holds in submodularity case, to compare our bound with it, we further refine our bound in this case.
If $f$ is submodular, the submodularity index $\lambda_f$ is strictly non-negative \cite{ZhouS16}; then we have $2f(S)+f(V \setminus S) \geq f(C)$, which immediately implies the following result.

\begin{corollary}
\label{cor:approximation_bound}
	If $f$ is submodular, then $f(S)\geq \frac{1}{3}f(C)$ or $f(V\setminus S)\geq \frac{1}{3}f(C)$.
\end{corollary}

Corollary~\ref{cor:approximation_bound} indicates that in submodularity case LS can achieve $\frac{1}{3}$-approximation on $f$.
Further, we can obtain an even better bound based on a mild assumption.
The assumption is that there exists a $B_i$ with $|B_i|>2$.
If this is true, then the solution $V \setminus S$ must be infeasible, i.e., $f(V \setminus S)=0$; therefore we have $2f(S) \geq f(C)$, which indicates that LS can achieve $\frac{1}{2}$-approximation on $f$, as stated in Corollary~\ref{cor:approximation_bound_2}.

\begin{corollary}
	\label{cor:approximation_bound_2}
	If $f$ is submodular and $\exists i \in [n]$ such that $|B_i|>2$, then $f(S)\geq \frac{1}{2}f(C)$.
\end{corollary}

\begin{figure}[t]
	\centering
	\subfloat{\includegraphics[width=.49\columnwidth]{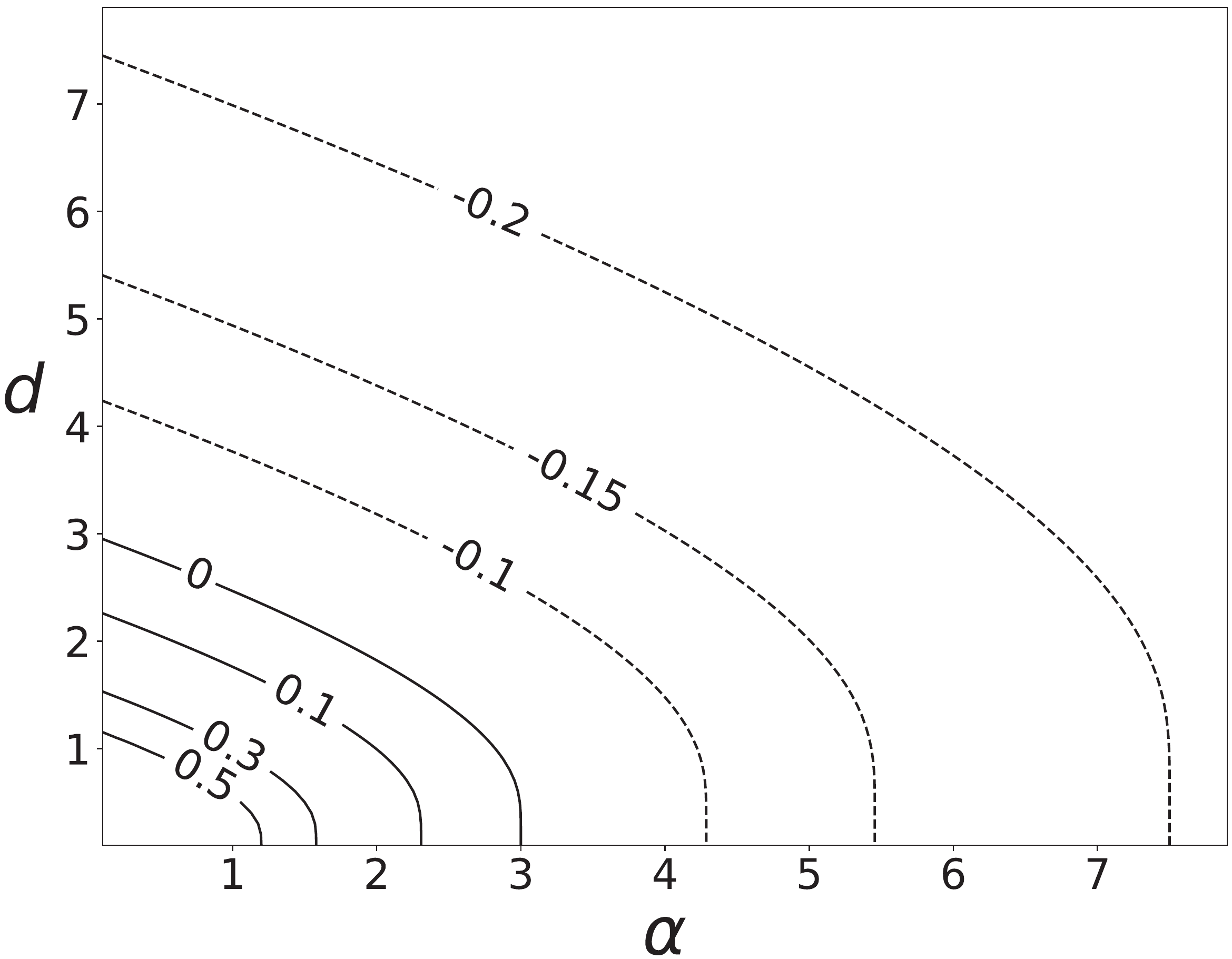}}
	\hfill
	\subfloat{\includegraphics[width=.49\columnwidth]{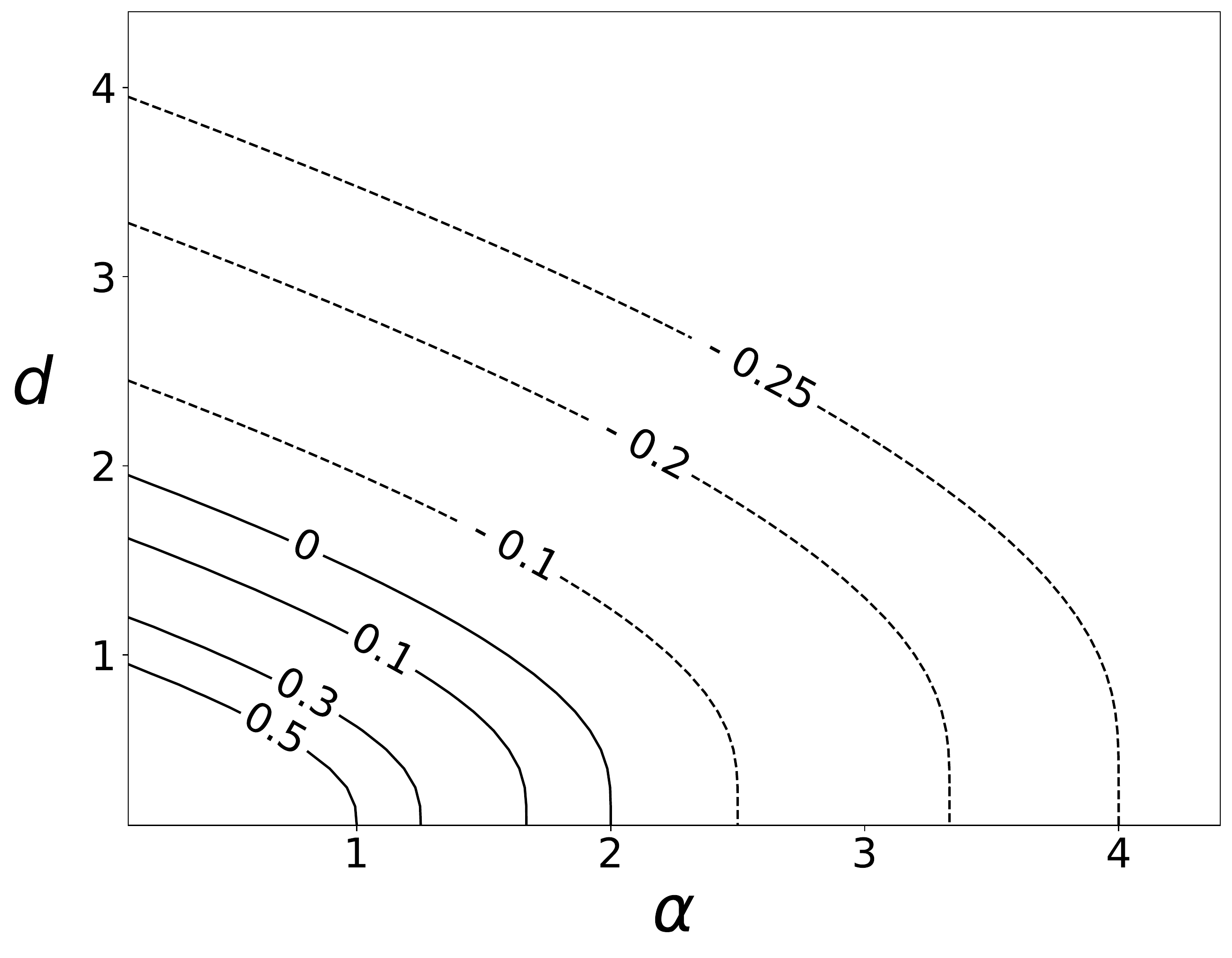}}
	\caption{Contour lines of the functions $\frac{1}{\alpha}(1-e^{-\alpha/d})-\frac{1}{3}$ (the left) and $\frac{1}{\alpha}(1-e^{-\alpha/d})-\frac{1}{2}$ (the right). Note negative value in the figures indicate that our approximation bounds are better than the previous bound.
	In the left figure, the function value is always negative when $d\geq 3$. In the right figure, the function value is always negative when $d\geq 2$.}
	\label{fig:compare_bound}
\end{figure}

We now compare the above bounds with the previous bound  $\frac{1}{\alpha}(1-e^{-\alpha \bar{d}/d})$ \cite{Friedrich2019}.
Considering the case where $d_i=1, \forall i \in [n]$, we have $\bar{d}=\max_{i \in [n]}d_i=1$, then the previous bound equals to $\frac{1}{\alpha}(1-e^{-\alpha/d})$, where $\alpha \geq 0$ and $d=\sum_{i \in [n]}d_i$.
Figure~\ref{fig:compare_bound} illustrates the contour lines of the functions $\frac{1}{\alpha}(1-e^{-\alpha/d})-\frac{1}{3}$ and $\frac{1}{\alpha}(1-e^{-\alpha/d})-\frac{1}{2}$.
It can be seen that in submodularity case our $\frac{1}{3}$-approximation and $\frac{1}{2}$-approximation bounds beat the previous bound with $d\geq3$ and $d\geq2$, respectively.

\subsection{Proof of Theorem 1}
We first introduce the definition of Submodularity Index (SmI) \cite{ZhouS16}, which measures the degree of submodularity.
Then we establish two lemmas (Lemma~\ref{lem:insert_delete} and Lemma~\ref{lem:deletion_exchange}) that relate SmI with the three conditions (insertion, deletion and exchange) for being a local optimum in Definition 3.
Finally, we prove Theorem~\ref{the:bound} based on these lemmas.

\begin{definition}
\label{def:smi}
	The submodularity index for a set function $f:2^V \rightarrow \mathbb{R}$, a set $L$, and a cardinality $k$, is defined as
	\[
	\lambda_{f}(L, k) \triangleq \min_{A \subseteq L \atop S \cap A=\emptyset,|S| \leq k} \varphi_{f}(S, A),
	\]
	where $\varphi_{f}(S, A) \triangleq \sum_{x \in S} f_{x}(A)-f_{S}(A)$ and $f_{S}(A) \triangleq f(A \cup S)-f(A)$.
\end{definition}

It is straightforward to verify that $\forall I \subseteq J$, SmI satisfies $\lambda_f(I,k) \geq \lambda_f(J,k)$.
The following two lemmas bound the degradation in submodularity with SmI, and the increase in $f$-value as the set size increases, respectively.

\begin{lemma}
	\label{lem:degradation_in_sub}
	Let $A$ be an arbitrary set, $B=A\cup\{y_1,...,y_M\}$ and $x \in \overline{B}$.
	Then $f_{x}(A)-f_{x}(B) \geq M\lambda_f(B,2)$.
\end{lemma}
\begin{proof}
	See Lemma 3 in \cite{ZhouS16}.
\end{proof}

\begin{lemma}
	\label{lem:bound_in_increase}
	Let $Y$ be an arbitrary set, $A \subseteq B$, and $Y \cap B = \varnothing$, then
	\begin{align*}
		f(A \cup Y)-f(A) 
		\geq &f(B \cup Y)-f(B)\\
		&+|B \setminus  A| \cdot|Y| \cdot \lambda_{f}(B \cup Y, 2).
	\end{align*}
\end{lemma}
\begin{proof}
See Lemma 4 in \cite{MoonAS19}.
\end{proof}

By the following lemma, we relate SmI with the first two conditions (insertion and deletion) for being a local optimum in Definition~\ref{def:local_optimum}.

\begin{lemma}
	\label{lem:insert_delete}
	Let $S$ be a feasible solution to the problem in Eq.~(\ref{eq:problem_definition}).
	If $f(S) \geq f(S\cup \{e\})$ for any $e \in V\setminus S$ and $f(S) \geq f(S\setminus \{e\})$ for any $e \in S$, then for any $I \subseteq S \subseteq J $, the following holds.
	\begin{align*}
		&f(I) \leq f(S)-\left(\begin{array}{c}
			|S \setminus  I| \\
			2
		\end{array}\right) \lambda_{f}(S, 2)\\
		&f(J) \leq f(S)-\left(\begin{array}{c}
			|J \setminus  S| \\
			2
		\end{array}\right) \lambda_{f}(J, 2).
	\end{align*}
\end{lemma}

\begin{proof}
	Let $I=I_{0} \subseteq I_{1} \subseteq \cdots \subseteq I_{k}=S$ be a chain of sets where $I_i \setminus I_{i-1}=\{a_i\}$.
	For $\forall i \in [k]$, by Lemma~\ref{lem:degradation_in_sub}, we have
	\begin{align*}
		f\left(I_{i}\right)-f\left(I_{i-1}\right) \geq f(S)&-f\left(S \setminus \left\{a_{i}\right\}\right)\\
		&+(k-i) \lambda_{f}\left(S \setminus \left\{a_{i}\right\}, 2\right).
	\end{align*}
	Since $f(S) \geq f(S\setminus \{e\})$ holds for any $e \in S$, by the property of SmI, we have
	\begin{align*}
		f\left(I_{i}\right)-f\left(I_{i-1}\right) \geq(k-i) \lambda_{f}(S, 2).
	\end{align*}
	By telescoping sum, it holds that
	\begin{align*}
		f({S})-f({I}) &\geq \sum_{i=1}^{k}(k-i) \lambda_{f}({S}, 2)\\
		&=\left(\begin{array}{c} |{S} \setminus  {I}| \\ 2 \end{array}\right) \lambda_{f}({S}, 2).
	\end{align*}
	Similarly, let $S=I_{0} \subseteq I_{1} \subseteq \cdots \subseteq I_{k}=J$ be a chain of sets where $I_i \setminus I_{i-1}=\{a_i\}$.
	For $\forall i \in [k]$, by Lemma~\ref{lem:degradation_in_sub}, we have
	\begin{align*}
		f\left(I_{i}\right)-f\left(I_{i-1}\right) \leq f(S \cup \{a_i\})&-f(S)\\
		&-(i-i) \lambda_{f}\left(I_{i-1}, 2\right).
	\end{align*}
	Since $f(S) \geq f(S\cup \{e\})$ holds for any $e \in V\setminus S$, by the property of SmI, we have
	\begin{align*}
		f\left(I_{i}\right)-f\left(I_{i-1}\right) \leq -(i-1) \lambda_{f}(J, 2).
	\end{align*}
	By telescoping sum, it holds that
	\begin{align*}
		f({J})-f({S}) &\leq -\sum_{i=1}^{k}(i-1) \lambda_{f}(J, 2)\\
		&=-\left(\begin{array}{c} |{J} \setminus  {S}| \\ 2 \end{array}\right) \lambda_{f}({J}, 2).
	\end{align*}
\end{proof}

Below we consider the third condition (exchange) for being a local optimum in Definition~\ref{def:local_optimum}.
We first introduce the well-known exchange property of matroids in Lemma~\ref{lem:exchange_property}.
Intuitively, this property states that for any two feasible solutions $I$ and $J$, we can add any element of $J$ to the set $I$ and remove at most one element from $I$ while keeping it feasible.
Moreover, each element of $I$ is allowed to be displaced by at most one element of $J$.

\begin{lemma}[Theorem~39.6 in \cite{schrijver2003combinatorial}]
	\label{lem:exchange_property}
	Let $I,J$ be two feasible solutions to the problem in Eq.~(1). Then there is a mapping $\pi:J \setminus I \rightarrow (I \setminus J) \cup \{\varnothing\}$ such that
	\begin{align*}
		\text{1) }&(I \setminus  \pi(b)) \cup\{b\} \text { is feasible for all } b \in J \setminus  I\\
		\text{2) }&\left|\pi^{-1}(e)\right| \leq 1 \text { for all } e \in I \setminus  J.
	\end{align*}
\end{lemma}

\begin{table*}[tbp]
	\centering
	\caption{Details of the textual attack approaches considered in the experiments.}
	\scalebox{0.88}{
	\begin{tabular}{ccccc}
		\toprule
		Attack Approach & Optimization Algorithm & \multicolumn{1}{c}{Objective Function} & Constraints & Word Substitution \\
		\midrule
		\makecell{PSO \\ Zang \textit{et al.}~\cite{ZangQYLZLS20}} & \makecell{Particle Swarm \\Optimization (PSO)} & Targeted & -  & \makecell{Sememe \\(HowNet word swap)} \\
		\midrule
		\makecell{GA \\ Alzantot \textit{et al.}~\cite{AlzantotSEHSC18}} & Genetic Algorithm (GA) & Targeted & \makecell{Percentage of words Perturbed,\\Language model perplexity,\\Word embedding distance} & \makecell{Counter-fitted word \\embedding swap} \\
		\midrule
		\makecell{pwws \\ Ren \textit{et al.}~\cite{RenDHC19}}  & \makecell{Saliency-based \\ Greedy Substitution (SBGS)} & \multicolumn{1}{c}{Untargeted} &   -    & \makecell{Synonym\\(WordNet word swap)} \\
		\midrule
		\makecell{morpheus \\ Tan \textit{et al.}~\cite{TanJKS20}} & \makecell{ Sequential Greedy \\ Substitution (SGS)} & Minimizing BLEU Score & -     & Infectional Morphology \\
		\midrule
		\makecell{TextFooler \\ Jin \textit{et al.}~\cite{JinJZS20}} & \makecell{Importance-based \\ Greedy Substitution (IBGS)} & Untargeted & \makecell{Word embedding distance,\\Part-of-speech match,\\USE sentence encoding cosine similarity} & \makecell{Counter-fitted word \\embedding swap} \\
		\midrule
		Ours  & Local Search (LS) & \makecell{Untargeted,\\Targeted,\\ Minimizing BLEU Score} & \multicolumn{2}{c}{Depending on the approach in which LS is embedded} \\
		\bottomrule
	\end{tabular}}
	\label{tab:approach_compare}
\end{table*}%

By the following lemma, we relate SmI with the second condition (deletion) and the third condition (exchange) for being a local optimum.
\begin{lemma}
	\label{lem:deletion_exchange}
	Let S be a feasible solution to the problem in Eq.~(1).
	If  $f(S) \geq f(S\setminus \{e\})$ for any $e \in S$,
	and $f(S) \geq f(S \setminus \{e\} \cup \{v\})$ for any $e \in S$ and $v \in V \setminus S$.
	Then for any feasible solution $C$, the following holds.
	\begin{align*}
		2f(S) \geq f(S \cup C) &+ f(S \cap C) +\\
		& \beta \lambda_{f}(V,2)+ \lambda_{f}(S,|C \setminus S|),
	\end{align*}
	where $\beta={|S \setminus C| \choose 2}+|C \setminus S|$.
\end{lemma}
\begin{proof}
	By definition of SmI, we have
	\begin{align*}
		f(&S \cup C) -f(S) \leq\\
		&\sum_{b \in C \setminus S}[f(S \cup \{b\})-f(S)] - \lambda_f(S, |C\setminus S|).
	\end{align*}
	Also, by Lemma~\ref{lem:degradation_in_sub},
	\begin{align*}
		f_b(S)\leq f_b(S\setminus \{\pi(b)\}) - \lambda_f(S,2).
	\end{align*}
	Therefore, it holds that:
	\begin{align*}
		&f(S \cup C) - f(S) \leq\\
		&\sum_{b \in C \setminus S}[f_b(S \setminus \{\pi(b)\})-\lambda_f(S,2)] - \lambda_f(S, |C\setminus S|).
	\end{align*}
	It remains to investigate  $\sum_{b \in C \setminus S}f_b(S \setminus \{\pi(b)\})$.
	Since for any $e \in S$ and $v \in V \setminus S$, $f(S) \geq f(S \setminus \{e\} \cup \{v\})$, it holds that $f_b(S \setminus \{\pi(b)\}) \leq f(S)-f(S \setminus\{\pi(b)\}).$
	Further, by Lemma~\ref{lem:exchange_property} and the fact $f(S)\geq f(S\setminus\{\pi(b)\})$,
	\begin{align*}
		\sum_{b \in C\setminus S} [f(S)-f(S \setminus\{\pi(b)\})] \leq 
		\sum_{b \in S \setminus C}[f(S)-f(S\setminus \{b\})].
	\end{align*}
	Let $I=S\cap C$, and Let $I=I_{0} \subseteq I_{1} \subseteq \cdots \subseteq I_{|S\setminus C|}=S$ be a chain of sets where $I_i \setminus I_{i-1}=\{a_i\}$.
	Then by Lemma~\ref{lem:degradation_in_sub},
	\begin{align*}
		f\left(I_{i}\right)&-f\left(I_{i-1}\right) \geq\\
		&f(S)-f\left(S \setminus \left\{a_{i}\right\}\right)+(|S\setminus C|-i) \lambda_{f}\left(S \setminus \left\{a_{i}\right\}, 2\right).
	\end{align*}
	By telescoping sum, 
	\begin{align*}
		f({S})-&f({S\cap C}) \geq \\
		&\sum_{b \in S \setminus C}[f(S)-f(S \setminus \{b\})] +
		\left(\begin{array}{c} |{S} \setminus  {C}| \\ 2 \end{array}\right) \lambda_{f}({S}, 2).
	\end{align*}
	Summing the above results, we have:
	\begin{align*}
		f&(S\cup C)-f(S) \leq f(S)-f(S \cap C)\\
		&-\left[ 
		\left( \left(\begin{array}{c} |{S} \setminus  {C}| \\ 2 \end{array}\right) + |C \setminus S|\right)\lambda_{f}({S}, 2) + \lambda_f(S,|C\setminus S|)
		\right].
	\end{align*}
	The proof is complete.
\end{proof}

\begin{table*}[tbp]
	\centering
	\caption{Details of the attack scenarios considered in the experiments, categorized by their source publications.
			``Baseline'' represents the original optimization algorithm.
			``\#AT. Samples'' represents the number of the samples for attacking.
			``Test Per.'' represents the testing performance of the victim models. For the task of question answering, ``Test Per.'' is the average F1 score. For machine translation, ``Test Per.'' is the BLEU score. For other tasks, ``Test Per.'' refers to accuracy. }
	\begin{tabular}{cccccc}
		\toprule
		Baseline & NLP Task & Dataset & Victim Model & Test Per. (\%) & \#AT. Samples \\
		\midrule
		\multirow{6}[6]{*}{\makecell {Particle Swarm \\ Optimization (PSO)\\ Zang \textit{et al.}~\cite{ZangQYLZLS20}}}  & \multirow{4}[4]{*}{Sentiment Analysis} & \multirow{2}[2]{*}{IMDB} & BiLSTM & 89.10  & 2719 \\
		&             &       & BERT  & 90.76 & 2707 \\
		\cmidrule{3-6}          &       &       \multirow{2}[2]{*}{SST-2} & BiLSTM & 83.75 & 1180 \\
		&             &       & BERT  & 90.28 & 1198 \\
		\cmidrule{2-6}                & \multirow{2}[2]{*}{Textual Entailment} & \multirow{2}[2]{*}{SNLI} & BiLSTM & 84.43 & 2210 \\
		&             &       & BERT  & 89.58 & 2287 \\
		\midrule
		\multirow{2}[2]{*}{\makecell{Genetic Algorithm (GA)\\ Alzantot \textit{et al.}~\cite{AlzantotSEHSC18}}}  & Sentiment Analysis & IMDB  & LSTM  & 88.55 & 2698 \\
		& Textual Entailment   & SNLI  & DNN   & 82.30  & 2163 \\
		\midrule
		\multirow{4}[4]{*}{\makecell {Saliency-based \\ Greedy Substitution (SBGS)\\Ren \textit{et al.}~\cite{RenDHC19}}} & \multirow{2}[2]{*}{Sentiment Analysis} & \multirow{2}[2]{*}{IMDB} & Word-CNN & 88.26 & 2162 \\
		&             &       & BiLSTM & 85.71 & 2112 \\
		\cmidrule{2-6}                & \multirow{2}[2]{*}{News Categorization} & \multirow{2}[2]{*}{AG's News} & Char-CNN & 89.36 & 6766 \\
		&             &       & Word-CNN & 90.76 & 6872 \\
		\midrule
		\multirow{8}[6]{*}{\makecell {Sequential Greedy \\ Substitution (SGS)\\Tan \textit{et al.}~\cite{TanJKS20}}}  & \multirow{6}[4]{*}{Question Answering} & \multirow{4}[2]{*}{\makecell{SQuAD 2.0 \\ (Answerable)}} & GloVe-BiDAF & 78.67 & 3501 \\
		&             &       & ELMo-BiDAF & 80.90  & 3572 \\
		&             &       & BERT  & 81.19 & 3405 \\
		&            &       & SpanBERT & 88.52 & 3458 \\
		\cmidrule{3-6}          &             & \multirow{2}[2]{*}{\makecell{SQuAD 2.0 \\ (All)}} & BERT  & 81.52 & 6411 \\
		&       &           & SpanBERT & 87.71 & 6586 \\
		\cmidrule{2-6}               & \multirow{2}[2]{*}{Machine Translation} & \multicolumn{1}{c}{\multirow{2}[2]{*}{\makecell{newstest2014 \\ (En-Fr)}}} & ConvS2S & 40.83 & 2298 \\
		&       &            & Transformer-big & 43.16 & 2340 \\
		\midrule
		\multirow{6}[0]{*}{\makecell {Importance-based \\ Greedy Substitution (IBGS) \\Jin \textit{et al.}~\cite{JinJZS20}}} &  \multirow{3}[0]{*}{Sentiment Analysis} & \multirow{3}[0]{*}{MR} & BERT  & 85.01 & 850 \\
		&             &       & Word-CNN & 78.02 & 766 \\
		&             &       & Word-LSTM & 80.72 & 786 \\
		\cmidrule{2-6}     & \multirow{3}[0]{*}{Textual Entailment} & \multirow{3}[0]{*}{SNLI} & BERT  & 90.48 & 2433 \\
		&             &       & Infersent & 83.38 & 2122 \\
		&             &       & ESIM  & 86.12 & 2307 \\
		\bottomrule
	\end{tabular}
	\label{tab:attack_settings}
\end{table*}

We now prove Theorem~\ref{the:bound} using the above lemmas.

\begin{proof}
	By Lemma~\ref{lem:deletion_exchange}, the following holds.
	\begin{align}
	\begin{split}
	\label{ieq:2}
		2f(S) \geq f(S \cup C) &+ f(S \cap C)\\
		& +\beta \lambda_{f}(V,2)+ \lambda_{f}(S,|C \setminus S|).
	\end{split}
	\end{align}
	Also by Lemma~\ref{lem:bound_in_increase}, we have
	\begin{align}
	\begin{split}
	\label{ieq:3}
		f(S&\cup C) + f(V \setminus S)\\
		&\geq f({C} \setminus  {S})+f({V})+|\overline{{S} \cup {C}}| \cdot|{S}| \cdot \lambda_{f}({V}, 2)\\
		& \geq f({C} \setminus  {S})+|\overline{{S} \cup {C}}| \cdot|{S}| \cdot \lambda_{f}({V}, 2)
	\end{split}
	\end{align}
	\begin{align}
	\begin{split}
	\label{ieq:4}
		f(S&\cap C) + f(C \setminus S)\\
		&\geq f({C})+f(\emptyset)+|{C} \setminus  {S}| \cdot|{S} \cap {C}| \cdot \lambda_{f}({C}, 2)\\
		& \geq f({C})+|{C} \setminus  {S}| \cdot|{S} \cap {C}| \cdot \lambda_{f}({C}, 2).
	\end{split}
	\end{align}
	Summing the inequalities~(\ref{ieq:2})-(\ref{ieq:4}), we have
	\begin{align*}
		2 f({S})+f({V} \setminus  {S}) \geq & f(C)+\beta \lambda_{f}({S}, 2) \\
		&+\lambda_{f}(S,|C \setminus S|)\\
		&+|\overline{{S} \cup {C}}| \cdot|{S}| \cdot \lambda_{f}({V}, 2)\\
		&+|{C} \setminus  {S}| \cdot|{S} \cap {C}| \cdot \lambda_{f}({C}, 2).
	\end{align*}
	By the fact $\lambda_f(\cdot,2)\geq \lambda_f(V,2)$,
	\begin{align}
	\begin{split}
	\label{ieq:5}
		2f(S)+f(V \setminus S) \geq f(C) +&\delta \lambda_f(V,2) +\\
									      &\lambda_{f}(S,|C \setminus S|),
	\end{split}
	\end{align}
	where $\delta$ is defined as in Theorem~\ref{the:bound}.
	For the other half, let $I=S\cap C$ and $J=S\cup C$ in Lemma~\ref{lem:insert_delete}, the following holds.
	\begin{align}
	\begin{split}
	\label{ieq:6}
		2f(S) &\geq f(S \cup C) + f(S \cap C) + \\
		& \left(\begin{array}{c}
			|{S} \setminus  {C}| \\
			2
		\end{array}\right) \lambda_{f}({S}, 2) +
		\left(\begin{array}{c}
			|{C} \setminus  {S}| \\
			2
		\end{array}\right) \lambda_{f}({S} \cup {C}, 2).
	\end{split}
	\end{align}
	Similarly, by summing the inequilities~(\ref{ieq:3})(\ref{ieq:4})(\ref{ieq:6}), we have
	\begin{align*}
	2 f({S})+f({V} \setminus  {S}) \geq & f(C)+\left(\begin{array}{c}
		|{S} \setminus  {C}| \\
		2
	\end{array}\right) \lambda_{f}({S}, 2)\\
	&+\left(\begin{array}{c}
		|{C} \setminus  {S}| \\
		2
	\end{array}\right) \lambda_{f}({S} \cup {C}, 2)\\
	&+|\overline{{S} \cup {C}}| \cdot|{S}| \cdot \lambda_{f}({V}, 2)\\
	&+|{C} \setminus  {S}| \cdot|{S} \cap {C}| \cdot \lambda_{f}({C}, 2).
	\end{align*}
	By the fact $\lambda_f(\cdot,2)\geq \lambda_f(V,2)$,
	\begin{align}
	\begin{split}
	\label{ieq:7}
		2f(S)+f(V \setminus S) \geq f(C) + \xi \lambda_f(V,2),
	\end{split}
	\end{align}
	where $\xi$ is defined as in Theorem~\ref{the:bound}.
	Based on inequilities~(\ref{ieq:5}) and (\ref{ieq:7}), we finally have
	\begin{align*}
		2f(S)+f(V\setminus S) & \geq  f(C) + \max \{ \xi \lambda_{f}(V,2),\\ &\delta \lambda_{f}(V,2)+ \lambda_{f}(S,|C \setminus S|) \}.
	\end{align*}
	The proof is complete.
\end{proof}

\section{Experiments}
\label{sec:exp}
We conduct extensive experiments to evaluate our algorithm in various attack scenarios.
Specifically, we repeat the settings considered by five recent open-source word-level attack approaches~\cite{AlzantotSEHSC18,RenDHC19, ZangQYLZLS20,TanJKS20,JinJZS20}, and compare LS with their original optimization algorithms.
These approaches are representative in the sense that they have employed different word-substitution methods and optimization algorithms, and have achieved the state-of-the-art attack performance for various NLP tasks and victim models.
In addition, as the baselines, the optimization algorithms adopted by them represent the two mainstream choices in the literature: population-based algorithms and simple heuristics.
Finally, the objective functions considered by them are also various, including the predicted probability on a specific wrong class (targeted attack) \cite{ZangQYLZLS20,AlzantotSEHSC18}, one minus the predicted probability on the ground truth (untargeted attack) \cite{RenDHC19,JinJZS20}, and the model's loss \cite{TanJKS20}.
Morris \textit{et al.} \cite{MorrisLYGJQ20} presented a comprehensive comparison of previous textual attack approaches in terms of objective functions, constraints, perturbation strategies and optimization algorithms.
Following \cite{MorrisLYGJQ20}, we present these details of all the considered word-level attack approaches in Table~\ref{tab:approach_compare}.
Note in the experiments, for each considered approach we replace its optimization algorithm with LS, and then compare it with the original approach.
We use the fine-tuned models, datasets and attack approach implementations provided in their code repositories.
All the codes, datasets, as well as the step-by-step instructions for repeating our experiments, are available at \url{https://github.com/ColinLu50/NLP-Attack-LocalSearch}.
\begin{table*}[tbp]
	\centering
	\caption{Success rates and average query number of LS and PSO in attack scenarios introduced by Zang \textit{et al.}~\cite{ZangQYLZLS20}.
		Each table cell contains two values, a success rate (\%) and an average query number, separated by ``\textbar''.
		For PSO, its average performance $\pm$ standard deviation across 10 repeated runs is reported.
		Note for success rate, the higher the better; for query number, the smaller the better.}
	\scalebox{1.0}
	{
		\begin{tabular}{ccccccc}
			\toprule
			\multicolumn{1}{c}{\multirow{2}[4]{*}{Alg.}} &
			\multicolumn{2}{c}{IMDB} & \multicolumn{2}{c}{SST-2} & \multicolumn{2}{c}{SNLI} \\
			\cmidrule(lr){2-3} \cmidrule(lr){4-5} \cmidrule(lr){6-7}
			& BiLSTM & BERT  & BiLSTM & BERT  & BiLSTM & BERT \\
			\midrule
			PSO   & 99.84$\pm$.04\textbar3613$\pm$14 & \textbf{94.78}$\pm$.44\textbar5375$\pm$32  &  93.71$\pm$.17\textbar1269$\pm$13  &  89.40$\pm$.07\textbar1642$\pm$11  & 72.83$\pm$.42\textbar2931$\pm$71  &  76.39$\pm$.20\textbar2151$\pm$32  \\
			LS & \textbf{99.93}\textbar\textbf{2220} & 94.16\textbar\textbf{4332} & \textbf{94.07}\textbar\textbf{295} & \textbf{89.57}\textbar\textbf{344} & \textbf{74.71}\textbar\textbf{246} & \textbf{76.43}\textbar\textbf{233} \\
			\bottomrule
	\end{tabular}}
	\label{tab:success_efficiency_pso}
\end{table*}

\begin{table*}[tbp]
	\centering
	\caption{Success rates and average query number of LS and SGS in attack scenarios introduced by Tan \textit{et al.}~\cite{TanJKS20}.}
	\scalebox{1.0}{
		\begin{tabular}{ccccccccc}
			\toprule
			\multirow{2}[2]{*}{Alg.} & \multicolumn{4}{c}{SQuAD 2.0 (Answerable)} & \multicolumn{2}{c}{SQuAD 2.0 (All)} & \multicolumn{2}{c}{newstest2014 (En-Fr)} \\
			\cmidrule(lr){2-5} \cmidrule(lr){6-7} \cmidrule(lr){8-9}
			& GloVe-BiDAF & ELMo-BiDAF & BERT  & SpanBERT & BERT  & SpanBERT & ConvS2S & Transformer-big \\
			\midrule
			SGS   & 13.94\textbar\textbf{32} & 11.28\textbar\textbf{33} & 18.27\textbar\textbf{32} & 14.03\textbar\textbf{33} & 14.87\textbar\textbf{31} & 10.70\textbar\textbf{33} & 44.13\textbar\textbf{53} & 38.25\textbar\textbf{55} \\
			LS    & \textbf{22.88}\textbar61 & \textbf{19.20}\textbar64 & \textbf{24.99}\textbar49 & \textbf{18.77}\textbar51 & \textbf{21.82}\textbar53 & \textbf{15.26}\textbar55 & \textbf{48.74}\textbar92 & \textbf{42.18}\textbar90 \\
			\bottomrule
	\end{tabular}}
	\label{tab:success_efficiency_mehpheus}
\end{table*}

\begin{table*}[tbp]
	\centering
	\caption{Success rates and average query number of LS and IBGS in attack scenarios introduced by Jin \textit{et al.}~\cite{JinJZS20}.}
	\begin{tabular}{ccccccc}
		\toprule
		\multirow{2}[4]{*}{Alg.} & \multicolumn{3}{c}{MR} & \multicolumn{3}{c}{SNLI} \\
		\cmidrule(lr){2-4} \cmidrule(lr){5-7}          & BERT  & WordLSTM & WordCNN & BERT  & Infersent & ESIM \\
		\midrule
		IBGS & 78.59\textbar\textbf{127.9} & 88.49\textbar\textbf{102.2} & 90.08\textbar\textbf{100.9} & 88.90\textbar\textbf{68.9} & 93.97\textbar\textbf{63.3} & 90.29\textbar\textbf{68.4} \\
		\midrule
		LS    & \textbf{92.94}\textbar613.7 & \textbf{94.50}\textbar564.9 & \textbf{95.95}\textbar566.2 &  \textbf{95.85}\textbar 278.5 & \textbf{98.65}\textbar258.7 & \textbf{97.44}\textbar286.3 \\
		\bottomrule
	\end{tabular}%
	\label{tab:success_efficiency_textfooler}%
\end{table*}%

\begin{table}[tbp]
	\centering
	\caption{Success rates and average query number of LS and SBGS in attack scenarios introduced by Ren \textit{et al.}~\cite{RenDHC19}.}
	\scalebox{1.0}
	{
		\begin{tabular}{ccccc}
			\toprule
			\multirow{2}[2]{*}{Alg.} & \multicolumn{2}{c}{IMDB} & \multicolumn{2}{c}{AG's News} \\
			\cmidrule(lr){2-3} \cmidrule(lr){4-5}
			& Word-CNN & BiLSTM & Char-CNN & Word-CNN \\
			\midrule
			SBGS  & 82.84\textbar\textbf{153} & 87.78\textbar\textbf{153} & 74.33\textbar\textbf{75}& 81.49\textbar\textbf{75}\\
			LS    & \textbf{85.62}\textbar599 & \textbf{92.28}\textbar579 & \textbf{82.62}\textbar165 & \textbf{87.54}\textbar173 \\
			\bottomrule
	\end{tabular}}
	\label{tab:success_efficiency_pwws}%
\end{table}

\subsection{Attack Scenarios}
Table~\ref{tab:attack_settings} summarizes all the 26 different attack scenarios (each row is a unique scenario) considered in the experiments,
which in total involve 5 NLP tasks, 8 datasets and 26 NLP models.
From the perspective of optimization, each scenario in Table~\ref{tab:attack_settings} represents a specific sub-class of instances of the optimization problem defined in Eq~(\ref{eq:problem_definition}).
Therefore, LS will be tested on 26 different sub-classes of problem instances, containing 75909 instances in total, which is expected to be sufficient for assessing its performance and robustness.

\subsection{Experimental Setup}
\label{sec:exp_setup}
\paragraph{Algorithm Settings}
For the baseline algorithms, we use their recommended hyper-parameter settings.
Specifically, for PSO \cite{ZangQYLZLS20}, $V_{max}$, $\omega_{max}$, $\omega_{min}$, $P_{max}$, $P_{min}$ and $k$ are set to 1, 0.8, 0.2, 0.8, 0.2 and 2, respectively.
For GA \cite{AlzantotSEHSC18}, $N$, $K$ and $\delta$ are set to 8, 4 and 0.5, respectively.
For both PSO and GA, the maximum number of iteration and the population size are set to 20 and 60.
As for the three heuristics and LS, all of them are free of hyper-parameters.

All the algorithms will immediately terminate once they achieve successful attacks, which means the model predicts the targeted label for targeted attack, or any wrong label for untargeted attack.
When maximizing the model's loss \cite{TanJKS20}, successful attacks mean the task's score (e.g., F1 and BLEU) becomes 0.

\begin{table}[tbp]
	\centering
	\caption{Success rates and average query number of LS and GA in attack scenarios introduced by Alzantot \textit{et al.}~\cite{AlzantotSEHSC18}.
		For GA, its average performance $\pm$ standard deviation across 10 repeated runs is reported.}
	\scalebox{1.0}{
		\begin{tabular}{ccc}
			\toprule
			Alg. & IMDB+LSTM  & SNLI+DNN \\
			\midrule
			GA    & 94.63$\pm$.54\textbar2257$\pm$89 & 33.77$\pm$1.08\textbar2989$\pm$41 \\
			LS    & \textbf{99.33}\textbar\textbf{665} & \textbf{42.81}\textbar\textbf{128} \\
			\bottomrule
	\end{tabular}}
	\label{tab:success_efficiency_ga}
\end{table}

\paragraph{Samples for Attacking}
It is conceivable that the shorter the original input is, the fewer words in it that can be substituted, and therefore the more difficult it is to craft a successful adversarial example.
For those too short inputs, it is possible that there exist no successful adversarial examples in the whole search spaces, in which cases comparing different algorithms is just meaningless.
On the other hand, those too long inputs are too easy to be successfully attacked since there exist so many words in them that can be substituted.
As a result, using these inputs for attacking cannot discriminate well between the search capabilities of different optimization algorithms.
Based on the above considerations, in the experiments we restrict the lengths of the original inputs to 10-100, following  \cite{AlzantotSEHSC18,ZangQYLZLS20}.
Also, we consider the adversarial examples with modification rates higher than 25\% as failed attacks.
Previous studies \cite{AlzantotSEHSC18, ZangQYLZLS20, RenDHC19} usually sample a number (typically 1000) of correctly classified instances from the test sets as the original input for attacking, which however may introduce \textit{selection bias} for assessing the attack performance.
To avoid this issue, we use \textit{all} the correctly classified testing instances for attacking, which usually leads to much more used samples compared to previous studies (see the last column of Table~\ref{tab:attack_settings}).
As aforementioned, in word-level textual attack, the objective function can be non-submodular (see Section~\ref{sec:problem_formulation}).
To illustrate this, we randomly select 100 samples from each of AG's news, SST-2 and IMDB datasets (their associated victim models are Word-CNN, BiLSTM and BERT, respectively), and run experiments to verify whether their objective functions are submodular.
The detailed experimental results are available at \url{https://github.com/ColinLu50/NLP-Attack-LocalSearch}. 
It is shown that only one example from SST-2 has submodular objective function.
This implies generally we can be almost sure that the objective function in word-level textual attack is non-submodular.

\paragraph{Evaluation Metrics}
Similar to previous studies, we use attack success rates to assess the attacking ability.
In addition, we also consider query number as the metric of attack efficiency.
Concretely, attack success rate refers to the percentage of successful attacks, and the query number is the number of queries consumed by the optimization algorithms for attacking an example.

\subsection{Results and Analysis}
Tables~\ref{tab:success_efficiency_pso}-\ref{tab:success_efficiency_ga}
present the attack success rates and the average query number of LS and the corresponding baselines in each attack scenario.
Considering PSO \cite{ZangQYLZLS20} and GA \cite{AlzantotSEHSC18} are both randomized algorithms, to make fair comparison, we run these two algorithms for 10 times, and report their average performance.

The first observation from these results is that generally LS obtains higher success rates than the baselines.
Actually, it performs better than the baselines in 25/26 scenarios.
Compared with GA (Table~\ref{tab:success_efficiency_ga}) and the three heuristics SBGS (Table~\ref{tab:success_efficiency_pwws}), IBGS (Table~\ref{tab:success_efficiency_textfooler}) and SGS (Table~\ref{tab:success_efficiency_mehpheus}), the advantage of LS in terms of success rates is particularly significant, ranging from 2.78\% to 9.04\%.
Although PSO can achieve nearly competitive success rates to LS, the latter performs much better than PSO (and GA) in terms of query efficiency.  
Notably, compared with these two population-based algorithms, in 6/8 scenarios LS dramatically reduces the query number by an order of magnitude.
Compared with SBGS, IBGS and SGS, LS requires more queries.
However, the amount of increase is usually reasonable (e.g., around 30 in Table~\ref{tab:success_efficiency_mehpheus}) and therefore acceptable, considering LS achieves significantly higher success rates.

One may notice that the attack success rates in Table~\ref{tab:success_efficiency_mehpheus} are much lower than others.
This is due to the fact that the settings considered by Tan \textit{et al.}~\cite{TanJKS20} (e.g., the criteria of successful attacks, see Section~\ref{sec:exp_setup}), are much more challenging.
Nevertheless, the high success rates obtained here on attacking popular models such as BERT and BiLSTM clearly demonstrate the vulnerability of DNNs.

\begin{figure}[tbp]
	\includegraphics[width=\columnwidth]{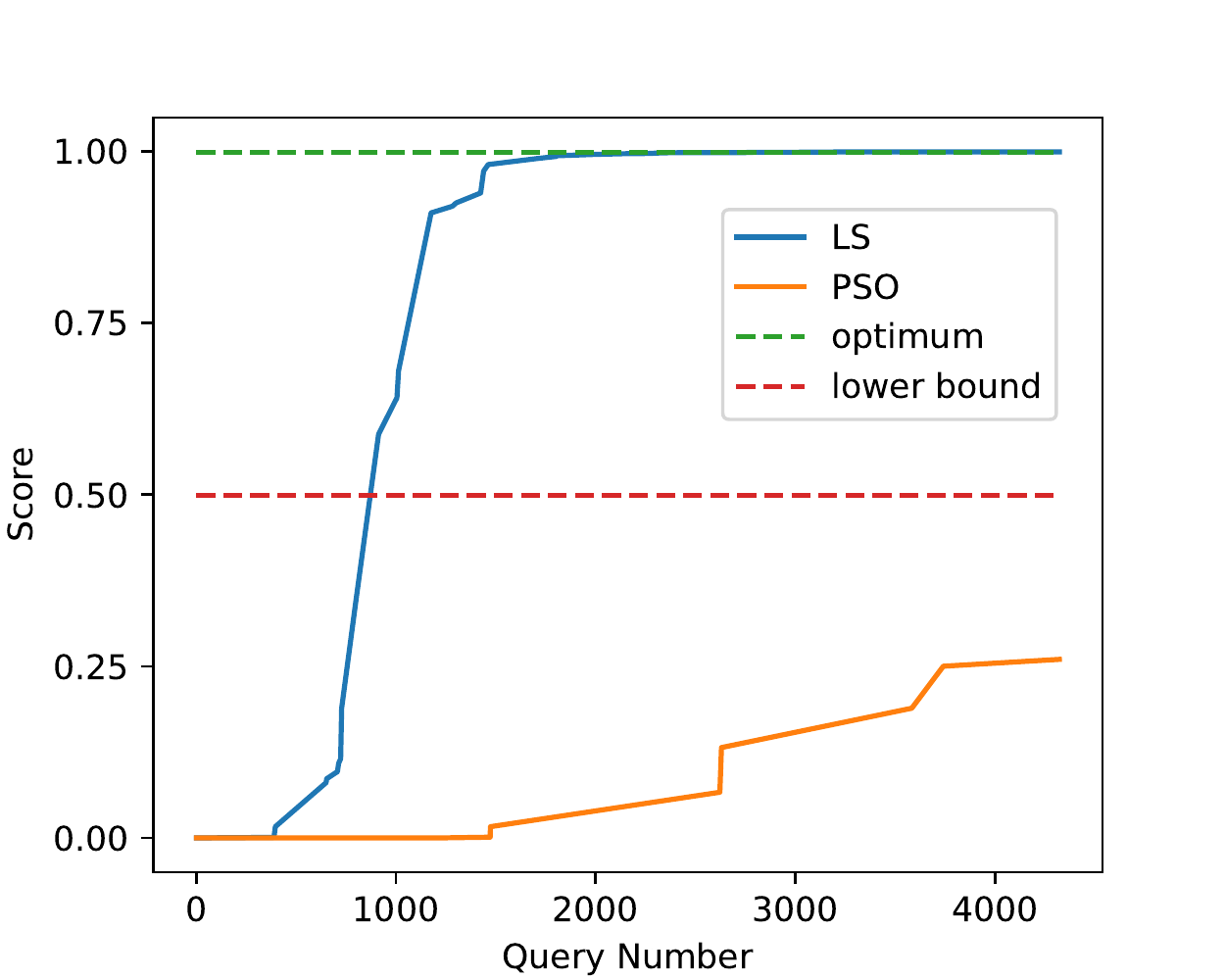}
	\caption{The optimum, the lower bound, as well as the scores of the solutions found by LS and PSO on an example from the SST-2 dataset.}
	\label{fig:relation}
\end{figure}

It is meaningful to reveal how the theoretical results established in this paper relate to the practical performance of LS.
Since Theorem~\ref{the:bound} states that the gap between the score obtained by LS and the optimum is strictly bounded, we randomly select an example from SST-2 dataset \footnote{Specifically, the example is ``Well written, nicely acted and beautifully shot and scored, the film works on several levels, openly questioning social mores while ensnaring the audience with its emotional pull.''} (the victim model is BERT) and enumerate all the possible adversarial examples in the search space to find the optimum and calculate the SmI, i.e., Submodularity Index (see Definition~\ref{def:smi}).
Then based on these results we further calculate the lower bound of the score obtained by LS according to Theorem 1, and run LS and PSO on this example until termination.
Finally, we illustrate the optimum, the lower bound, as well as the scores of the solutions found by LS and PSO in Figure~\ref{fig:relation}.
The first observation from Figure~\ref{fig:relation} is that the lower bound is very close to 0.5.
This implies that LS will almost surely find a successful adversarial example because it is guaranteed to obtain a score not smaller than the lower bound.
Indeed, LS achieves this goal and even approaches the optimum.
In comparison, PSO lacks such guarantee and its performance can be arbitrarily bad.
This sheds some light on the results of previous experiments that LS consistently achieves significantly higher attack success rates than the baselines.
That is to say, Theorem 1 provides the \textit{worst-case} performance guarantee for LS, while in practice LS can usually achieve much better scores than the theoretical lower bound.

\begin{table*}[tbp]
	\centering
	\caption{Evaluation results in terms of modification rate (Modi.), grammaticality (IPW) and fluency (PPL), in the scenarios introduced by Zang \textit{et al.}~\cite{ZangQYLZLS20}.
		Note for all these three metrics, the smaller the better.}
	\scalebox{1.0}
	{
		\begin{tabular}{ccccccccccc}
			\toprule
			\multirow{2}[4]{*}{\shortstack{Victim \\ Model}} & \multirow{2}[4]{*}{Alg.} & \multicolumn{3}{c}{IMDB} & \multicolumn{3}{c}{SST-2} & \multicolumn{3}{c}{SNLI} \\
			\cmidrule(lr){3-5} \cmidrule(lr){6-8} \cmidrule(lr){9-11}          &       & Modi. (\%)   & IPW ($\text{E-}{4}$) & PPL   & Modi. (\%)   & IPW ($\text{E-}{4}$) & PPL   & Modi. (\%)   & IPW ($\text{E-}{4}$) & PPL \\
			\midrule
			\multirow{2}[2]{*}{BiLSTM} & PSO   & 4.28  & 2.286 & 96.63 & 8.99  & \textbf{0.930} & 275.58 & \textbf{11.76} & 5.512 & 210.37 \\
			& LS    & \textbf{3.50} & \textbf{1.361} & \textbf{88.74} & \textbf{8.92} & 1.851 & \textbf{263.89} & 11.85 & \textbf{4.841} & \textbf{206.94} \\
			\midrule
			\multirow{2}[2]{*}{BERT} & PSO   & 4.12  & 2.208 & 99.04 & 8.12  & 4.292 & 279.82 & 11.69 & \textbf{2.285} & 215.65 \\
			& LS    & \textbf{3.83} & \textbf{1.803} & \textbf{96.18} & \textbf{7.88} & \textbf{3.849} & \textbf{272.26} & \textbf{11.67} & 4.561 & \textbf{208.89} \\
			\bottomrule
		\end{tabular}%
	}
	\label{tab:quality_PSO}%
\end{table*}
\begin{table}[tbp]
	\centering
	\caption{Human evaluation results of validity and perceptibility of the crafted adversarial examples on attacking BERT on SST-2, with sememe-based word substitution.
		The second row also lists the evaluation results of the original input. 
		``Valid.'' is the percentage of valid adversarial examples.
		``Percep.'' is the average perceptibility score. For both of them, the higher the better.}
	\scalebox{1.0}{
		\begin{tabular}{cccc}
			\toprule
			Victim Model & Alg.  & Valid. (\%) & Percep. \\
			\midrule
			N/A   & Ori. Input & 87.0    & 3.46 \\
			\midrule
			\multirow{2}[2]{*}{BERT} & PSO   & \textbf{79.0}    & 2.64 \\
			& LS    & 78.0    & \textbf{2.69} \\
			\bottomrule
	\end{tabular}}
	\label{tab:human_evaluation}
\end{table}

\subsection{Adversarial Example Quality and Validity}
We further evaluate the quality, validity and naturality of the crafted adversarial examples.
The validity (evaluated by human) is measured by the percentage of successful adversarial examples with unchanged \textit{true} labels.
As for adversarial example quality, similar to \cite{ZangQYLZLS20}, we consider modification rate, fluency, grammaticality, and perceptibility.
Concretely, modification rate is the percentage of words in the adversarial examples that differ from the original input; fluency is measured by the perplexity (PPL) of GPT-2 \cite{radford2019language}.
For grammaticality, we obtain the increase in number of grammatical error of adversarial examples compared to the original input, with the help of LanguageTool (\url{https://languagetool.org/}), and then divide it by the length of original input to calculate the \textit{amount of increase per word processed} (IPW).
Unlike the metric used by \cite{ZangQYLZLS20} which measures the increase rate of grammical error, IPW measures the average increase of grammatical error as the attack approach processing one input word.
Finally, the perceptibility refers to how indistinguishable the adversarial examples are from human-written texts.

Table~\ref{tab:quality_PSO} presents the evaluation results in terms of modification rate, fluency (PPL), and grammaticality (IPW), of the adversarial examples crafted in the scenarios introduced by Zang \textit{et al.}~\cite{ZangQYLZLS20}.
Results in other scenarios are similar, and are given in Appendix~A.
The results in Table~\ref{tab:quality_PSO} indicate that in most scenarios the adversarial examples crafted by LS are better than that crafted by the baseline.
It is worth mentioning the IPW in Table~\ref{tab:quality_PSO} is typically of the order $10^{-4}$, which means the attack approach introduces roughly 1 grammatical error for every 10000 input words processed.

We perform human evaluation on 100 adversarial examples crafted by PSO and LS respectively, for attacking BERT on SST-2 with sememe-based word substitution.
We ask three movie fans (volunteers) to make a binary sentiment classification (i.e., labeling it as ``positive'' or ``negative''), and give a perceptibility score chosen from $\{0,1,2,3,4\}$ which indicates ``Machine-generated'', ``More like Machine-generated'', ``Not Sure'', ``More like Human-written'' and ``Human-written'', respectively.
The final sentiment labels are determined by majority voting, and the final perceptibility scores are determined by averaging.
Table~\ref{tab:human_evaluation} presents the results.
Overall the adversarial examples crafted by LS are comparable to that crafted by PSO, and both of them are inferior to original human-authored input.
Nevertheless, based on human evaluation, overall the crafted examples (with average perceptibility score larger than 2) are more like human-written than machine-generated.
Table~\ref{tab:cases} displays some adversarial examples crafted by LS and the baselines on IMDB dataset, with different word-substitution methods.

\begin{table}[tbp]
	\centering
	\caption{The classification accuracy of the transferred adversarial examples crafted by PSO and LS in the scenarios considered by Zang \textit{et al.}~\cite{ZangQYLZLS20}.}
	\label{tab:transferability}%
	\scalebox{1.0}{
		\begin{tabular}{ccccc}
			\toprule
			Transfer & Alg. & IMDB  & SST-2 & SNLI \\
			\midrule
			\multirow{2}[2]{*}{BiLSTM $\Rightarrow$ BERT} & PSO   & 80.70  & 69.47 & 61.40 \\
			& LS    & \textbf{77.48} & \textbf{64.32} & \textbf{58.93} \\
			\midrule
			\multirow{2}[2]{*}{BERT $\Rightarrow$ BiLSTM} & PSO   & 82.85 & 63.43 & 51.60 \\
			& LS    & \textbf{79.40} & \textbf{59.06} & \textbf{49.71} \\
			\bottomrule
	\end{tabular}}
\end{table}

\subsection{Transferability}
The transferability of an adversarial example refers to its ability to attack other unseen models \cite{GoodfellowSS15}.
We evaluate transferability in the attack scenarios introduced by Zang \textit{et al.}~\cite{ZangQYLZLS20}, where the used sememe-based word-substitution often leads to better transferability than other word-substitution methods.
Specifically, on each dataset, we use BERT to classify the adversarial examples crafted for attacking BiLSTM, and vice versa.
Table~\ref{tab:transferability} presents the classification results on the adversarial examples crafted by LS and PSO.
Note lower classification accuracy indicates better transferability.
In can be observed that the adversarial examples crafted by LS generally exhibit better transferability than that crafted by PSO.

\begin{table}[tbp]
	\centering
	\caption{The decrease in attack success rates (\%) after retraining, when attacking BiLSTM on SST-2 with sememe-based word substitution.
		``Alg.'' and ``Adv. T'' represent the optimization algorithm for attacking and adversarial training, respectively.
		``w/o'' means the (original) attack success rates without retraining.}
	\scalebox{1.0}
	{
		\begin{tabular}{cccc}
			\toprule
			Alg. \textbackslash{} Adv. T & w/o  & PSO   & LS \\
			\midrule
			PSO   & 93.71 & -2.29 & -3.18 \\
			LS    & 94.07 & -2.28 & -2.39 \\
			\bottomrule
	\end{tabular}}
	\label{tab:adversarial_training_PSO}%
\end{table}
\begin{table}[tbp]
	\centering
	\caption{The decrease in attack success rates (\%) after retraining, when attacking BERT on IMDB with synonym-based word substitution.}
	\scalebox{1.0}
	{
		\begin{tabular}{cccc}
			\toprule
			Alg. \textbackslash{} Adv. T & w/o  & SBGS  & LS \\
			\midrule
			SBGS  & 82.84 & -7.40 & -8.68 \\
			LS    & 85.62 & -5.45 & -7.46 \\
			\bottomrule
	\end{tabular}}
	\label{tab:adversarial_training_sbGS}
\end{table}
\begin{table*}[tbp]
	\centering
	\caption{Some adversarial examples crafted by LS and the baselines on IMDB dataset, with different word-substitution methods.}
	\begin{tabular}{c}
		\toprule[.3em]
		\multicolumn{1}{c}{IMDB Example 1} \\
		\midrule[.2em]
		\multicolumn{1}{l}{\textbf{Original Input} (Prediction=Negative)} \\
		\midrule
		\multicolumn{1}{l}{\shortstack[l]{I'm normally a fan of Mel Gibson, but in this case he did a movie with a poor script. The acting for the most\\ part really wasn't that bad,
				but the story was just pointless with flaws and boring.}} \\
		\midrule[.2em]
		\multicolumn{1}{l}{\textbf{Synonym-based Word Substitution + LS} (Prediction=Positive):} \\
		\midrule
		\multicolumn{1}{l}{\shortstack[l]{I'm normally a fan of Mel Gibson, but in this case he did a movie with a \textbf{short} script. The acting for the most\\ part really wasn't that \textbf{tough}, but the story was just \textbf{superfluous} with flaws and boring.}} \\
		\midrule
		\multicolumn{1}{l}{\textbf{Synonym-based Word Substitution + SBGS} (Prediction=Positive):} \\
		\midrule
		\multicolumn{1}{l}{\shortstack[l]{I'm normally a fan of \textbf{David} Gibson, but in this \textbf{type} he did a movie with a \textbf{short} script. The acting for the most\\ part really wasn't that \textbf{tough},
				but the story was just \textbf{superfluous} with flaws and boring.}} \\
		\midrule[.2em]
		\multicolumn{1}{l}{\textbf{Counter-fitted Word Embedding Substitution + LS} (Prediction=Positive):} \\
		\midrule
		\multicolumn{1}{l}{\shortstack[l]{I'm normally a \textbf{admirer} of Mel Gibson, but in this case he did a movie with a \textbf{needy} script. The acting for the most\\ part really wasn't that bad,
				but the story was \textbf{righteous} \textbf{vain} with flaws and boring.}} \\
		\midrule
		\multicolumn{1}{l}{\textbf{Counter-fitted Word Embedding Substitution + GA} (Prediction=Positive):} \\
		\midrule
		\multicolumn{1}{l}{\shortstack[l]{I'm normally a \textbf{admirer} of Mel Gibson, but in this case he did a movie with a \textbf{needy} script. The acting for the \textbf{longer}\\ part really wasn't that \textbf{amiss},
				but the \textbf{tales} was \textbf{righteous} \textbf{unnecessary} with faults and boring.}}\\
		\midrule[.2em]
		\multicolumn{1}{l}{\textbf{Sememe-based Word Substituion + LS} (Prediction=Positive):} \\
		\midrule
		\multicolumn{1}{l}{\shortstack[l]{I'm normally a fan of Mel Gibson, but in this \textbf{posture} he \textbf{forged} a movie with a \textbf{wrenching} script. The acting for the most\\ part really wasn't that bad,
				but the story was just \textbf{rash} with flaws and boring.}} \\
		\midrule
		\multicolumn{1}{l}{\textbf{Sememe-based Word Substituion + PSO} (Prediction=Positive):} \\
		\midrule
		\multicolumn{1}{l}{\shortstack[l]{I'm \textbf{overall} a fan of Mel Gibson, but in this case he \textbf{forged} a movie with a \textbf{wrenching} script. The acting for the most\\ part really wasn't that bad,
				but the story was \textbf{likely} \textbf{rash} with \textbf{vulnerabilities} and boring.}} \\
		\midrule[.2em]
		\\
		\midrule[.3em]
		\multicolumn{1}{c}{IMDB Example 2} \\
		\midrule[.2em]
		\multicolumn{1}{l}{\textbf{Original Input} (Prediction=Negative)} \\
		\midrule
		\multicolumn{1}{l}{\shortstack[l]{People, please don't bother to watch this movie! This movie is bad! It's totally waste of time. I don't see any point here.\\ It's a stupid film with lousy plot and the acting is poor. I rather get myself beaten than watch this movie ever again.}} \\
		\midrule[.2em]
		\multicolumn{1}{l}{\textbf{Synonym-based Word Substitution + LS} (Prediction=Positive):} \\
		\midrule
		\multicolumn{1}{l}{\shortstack[l]{People, please don't \textbf{get} to watch this movie! This movie is \textbf{tough}! It's totally \textbf{wastefulness} of time. I don't see any point here.\\ It 's a stupid film with \textbf{dirty} plot and the acting is \textbf{short}. I rather get myself beaten than watch this movie ever again.}} \\
		\midrule
		\multicolumn{1}{l}{\textbf{Synonym-based Word Substitution + SBGS} (Prediction=Positive):} \\
		\midrule
		\multicolumn{1}{l}{\shortstack[l]{People, please don't \textbf{get} to \textbf{catch} this movie! This movie is tough! It's totally \textbf{wasteland} of time. I don't see any \textbf{degree} here.\\ It's a stupid film with \textbf{dirty} plot and the acting is \textbf{short}. I rather get myself beaten than watch this movie ever again.}} \\
		\midrule[.2em]
		\multicolumn{1}{l}{\textbf{Counter-fitted Word Embedding Substitution + LS} (Prediction=Positive):} \\
		\midrule
		\multicolumn{1}{l}{\shortstack[l]{People, please don't \textbf{irritate} to watch this movie! This movie is \textbf{wicked}! It's totally \textbf{litter} of time. I don't see any point here.\\ It's a stupid film with \textbf{miserable} plot and the acting is \textbf{poorer}. I rather get myself beaten than watch this movie ever again.}} \\
		\midrule
		\multicolumn{1}{l}{\textbf{Counter-fitted Word Embedding Substitution + GA} (Prediction=Positive):} \\
		\midrule
		\multicolumn{1}{l}{\shortstack[l]{People, \textbf{invited} don't \textbf{irritate} to watch this movie. This movie is \textbf{naughty}! It's \textbf{abundantly} \textbf{debris} of time. I don't see any point here.\\ It's a \textbf{daft} film with \textbf{miserable} plot and the acting is \textbf{poorer}. I rather get myself \textbf{pummeled} than watch this movie ever again.}}\\
		\midrule[.2em]
		\multicolumn{1}{l}{\textbf{Sememe-based Word Substituion + LS} (Prediction=Positive):} \\
		\midrule
		\multicolumn{1}{l}{\shortstack[l]{People, please don't bother to watch this movie. This movie is bad! It's totally waste of time. I don't see any point here.\\ It's a \textbf{bookish snapshot} with \textbf{spongy} plot and the acting is \textbf{wrenching}. I rather get myself beaten than watch this movie ever again.}} \\
		\midrule
		\multicolumn{1}{l}{\textbf{Sememe-based Word Substituion + PSO} (Prediction=Positive):} \\
		\midrule
		\multicolumn{1}{l}{\shortstack[l]{People, please don't bother to watch this movie. This movie is bad. It's \textbf{decidedly} waste of time. I don't see any point here.\\ It's a stupid \textbf{snapshot} with \textbf{off} plot and the acting is \textbf{wrenching}. I rather get myself beaten than watch this \textbf{polka} ever again.}} \\
		\bottomrule[.2em]
	\end{tabular}
	\label{tab:cases}
\end{table*}
\subsection{Adversarial Training}
By incorporating adversarial examples into the training process, adversarial training is aimed at improving the robustness of victim models \cite{GoodfellowSS15}.
Here we evaluate the robustness improvement in two attack scenarios.
Under the first scenario introduced by Zang \textit{et al.}~\cite{ZangQYLZLS20} where sememe-based word substitution is used, we use LS and PSO to craft 553 adversarial examples respectively (8\% of the original training set size) by attacking BiLSTM on the \textit{training set} of SST-2.
Then we include the adversarial examples into the training set and retrain a BiLSTM.
Once again, we use LS and PSO to attack this new model and obtain the attack success rates.
Finally, we assess the robustness improvement as the decrease in success rates after retraining.
The larger the decrease, the larger the robustness improvement.
The above whole procedure is repeated in the second scenario introduced by Ren \textit{et al.}~\cite{RenDHC19} with IMDB dataset and BERT, where synonym-based word substitution is used and the baseline algorithm is SBGS.
Table~\ref{tab:adversarial_training_PSO} and Table~\ref{tab:adversarial_training_sbGS} present the results.
It can be observed that adversarial training indeed makes the models more robust to adversarial attacks.
Moreover, the adversarial examples crafted by LS generally bring larger robustness improvement than that crafted by the baselines.
Further, as the optimization algorithm for textual attack, LS is also less affected by adversarial training compared to the baselines.

\section{Conclusion and Discussion}
\label{sec:conclusion}
In this paper, we propose a local search algorithm to solve the combinatorial optimization problem induced in the optimization step of word-level textual attack.
We prove its approximation bound, which is also the first general performance guarantee in the literature.
We conduct exhaustive experiments to demonstrate the superiority of our algorithm in terms of attack success rate, query efficiency, adversarial example quality, transferability and robustness improvement to victim models by adversarial training.

The algorithm proposed in this paper can be integrated into a wide range of textual attack frameworks, to help improve their query efficiency and success rates.
As a result, our algorithm can help better find the vulnerability, unfairness and safety issues in NLP models, and therefore help mitigate them in real-world applications.
On the other hand, there exists potential risk that the textual attack approaches can be used purposefully to attack the NLP models deployed in critical applications.
It is therefore necessary to study how to defend against such attacks.
In addition, in industry, this work also has an impact on solving the set maximization problem with partition matroids, which typically finds many applications such as service composition in cloud computing \cite{LiuWTQY15} and social welfare in combinatorial auctions \cite{Friedrich2019}.

There are several important future directions.
The first is to further improve the query efficiency of the algorithm, especially in cases of attacking long sentences.
The second is to extend the algorithm to the \textit{decision-based} setting where only the top labels predicted by the victim model are available.
The third future direction is to extend the algorithm to other languages, e.g., Chinese.
Although the general idea of word substitution as combinatorial optimization is widely applicable to different languages, it is still necessary to customize the method when applied to a specific language.
For example, in Chinese texts the words are not separated from each other; it is therefore necessary to conduct word segmentation before word substitution.
Another interesting future work is to improve the performance of existing search algorithms (e.g., PSO and GA), by integrating 
the proposed algorithmic components such as the three types of perturbations and the termination condition into them, such that they will be assured to find local optimal solutions.
Finally, it is also interesting to investigate how to automatically build an ensemble of attacks \cite{LiuTL020,LiuT019,LiuT020,TangLYY21} to reliably evaluate the adversarial robustness of NLP models.

\section*{Acknowledgment}
This work was supported in part by the Guangdong Provincial Key Laboratory under Grant 2020B121201001;
in part by the Program for Guangdong Introducing Innovative and Entrepreneurial Teams under Grant 2017ZT07X386;
in part by the Science and Technology Commission of Shanghai Municipality under Grant 19511120600; in part by the National Leading Youth Talent Support Program of China; and in part by the MOE University Scientific-Technological Innovation Plan Program.

\bibliographystyle{IEEEtran}
\bibliography{IEEEabrv,mybib}

\appendices
\begin{table*}[tbp]
	\centering
	\caption{Evaluation results in terms of modification rate (Modi.), grammaticality (IPW) and fluency (PPL), in the scenarios introduced by Alzantot \textit{et al.}~\cite{AlzantotSEHSC18}. Note for all these three metrics, the smaller the better.}
	\scalebox{1.0}{
		\begin{tabular}{cccccccc}
			\toprule
			\multirow{2}[4]{*}{Victim Model} & \multirow{2}[4]{*}{Alg.} & \multicolumn{3}{c}{IMDB} & \multicolumn{3}{c}{SNLI} \\
			\cmidrule(lr){3-5} \cmidrule(lr){6-8}          &       & Modi. (\%) & IPW ($\mathrm{E}$-04) & PPL   & Modi. (\%) & IPW ($\mathrm{E}$-04) & PPL \\
			\midrule
			\multirow{2}[2]{*}{LSTM} & GA    & 8.00     & 2.505 & 169.19 & \multicolumn{3}{c}{\multirow{2}[2]{*}{-}} \\
			& LS    & \textbf{5.02} & \textbf{0.969} & \textbf{145.54} & \multicolumn{3}{c}{} \\
			\midrule
			\multirow{2}[2]{*}{DNN} & GA    & \multicolumn{3}{c}{\multirow{2}[2]{*}{-}} & \textbf{14.19} & \textbf{6.863} & 257.50 \\
			& LS    & \multicolumn{3}{c}{}  & 14.20  & 8.470  & \textbf{251.93} \\
			\bottomrule
	\end{tabular}}
	\label{tab:quality_GA}
\end{table*}
\begin{table*}[tbp]
	\centering
	\caption{Evaluation results in terms of modification rate (Modi.), grammaticality (IPW) and fluency (PPL), in the scenarios introduced by Ren \textit{et al.}~\cite{RenDHC19}.}
	\scalebox{1.0}{
		\begin{tabular}{cccccccc}
			\toprule
			\multirow{2}[4]{*}{Victim Model} & \multirow{2}[4]{*}{Alg.} & \multicolumn{3}{c}{IMDB} & \multicolumn{3}{c}{AG's News} \\
			\cmidrule(lr){3-5} \cmidrule(lr){6-8}          &       & Modi. (\%) & IPW ($\mathrm{E}$-04) & PPL   & Modi. (\%) & IPW ($\mathrm{E}$-04) & PPL \\
			\midrule
			\multirow{2}[2]{*}{BiLSTM} & SBGS  & 6.41  & 2.524 & 96.54 & \multicolumn{3}{c}{\multirow{2}[2]{*}{-}} \\
			& LS    & \textbf{5.99} & \textbf{2.399} & \textbf{94.50} & \multicolumn{3}{c}{} \\
			\midrule
			\multirow{2}[2]{*}{Word-CNN} & SBGS  & 5.93  & 1.351 & 92.38 & 5.43  & 1.429 & 141.99 \\
			& LS    & \textbf{5.26} & \textbf{0.555} & \textbf{89.23} & \textbf{5.32} & \textbf{0.564} & \textbf{140.98} \\
			\midrule
			\multirow{2}[2]{*}{Char-CNN} & SBGS  & \multicolumn{3}{c}{\multirow{2}[2]{*}{-}} & 6.14  & 1.633 & 151.17 \\
			& LS    & \multicolumn{3}{c}{}  & \textbf{5.60} & \textbf{0.787} & \textbf{147.23} \\
			\bottomrule
	\end{tabular}}
	\label{tab:quality_pwws}
\end{table*}
\begin{table*}[tbp]
	\centering
	\caption{Evaluation results in terms of modification rate (Modi.), grammaticality (IPW) and fluency (PPL), in the scenarios introduced by Jin \textit{et al.}~\cite{JinJZS20}.}
	\begin{tabular}{ccccccccccc}
		\toprule
		\multirow{2}[4]{*}{Dataset} & \multirow{2}[4]{*}{Alg.} & \multicolumn{3}{c}{BERT} & \multicolumn{3}{c}{Word-LSTM} & \multicolumn{3}{c}{Word-CNN} \\
		\cmidrule(lr){3-5} \cmidrule(lr){6-8} \cmidrule(lr){9-11}          &       & Modi. (\%) & IPW (E-2) & PPL   & Modi. (\%) & IPW (E-2) & PPL   & Modi. (\%) & IPW (E-2) & PPL \\
		\midrule
		\multirow{2}[2]{*}{MR} & IBGS & 9.99  & 0.786 & \textbf{407.9} & 10.63 & 0.911 & 409.8 & 10.57 & 0.904 & \textbf{468.8} \\
		& LS    & \textbf{8.63} & \textbf{0.490} & 533.3 & \textbf{8.78} & \textbf{0.591} & \textbf{400.1} & \textbf{8.84} & \textbf{0.874} & 476.5 \\
		\toprule
		\multirow{2}[4]{*}{Dataset} & \multirow{2}[4]{*}{Alg.} & \multicolumn{3}{c}{BERT} & \multicolumn{3}{c}{Infersent} & \multicolumn{3}{c}{ESIM} \\
		\cmidrule(lr){3-5} \cmidrule(lr){6-8} \cmidrule(lr){9-11}          &       & Modi. (\%) & IPW (E-2) & PPL   & Modi. (\%) & IPW (E-2) & PPL   & Modi. (\%) & IPW (E-2) & PPL \\
		\midrule
		\multirow{2}[2]{*}{SNLI} & IBGS & 11.05 & 2.402 & 329.7 & 11.25 & \textbf{2.611} & \textbf{310.9} & 11.22 & 2.357 & 328.1 \\
		& LS    & \textbf{9.78} & \textbf{2.196} & \textbf{299.9} & \textbf{9.66} & 3.281 & 316.2 & \textbf{10.21} & \textbf{1.803} & \textbf{326.7} \\
		\bottomrule
	\end{tabular}%
	\label{tab:quality_textfooler}%
\end{table*}%
\begin{table*}[tbp]
	\centering
	\caption{Success rates and average query number of LS and the Greedy Algorithm in attack scenarios introduced by Zang \textit{et al.}~\cite{ZangQYLZLS20}.
	Each table cell contains two values, a success rate (\%) and an average query number, separated by ``\textbar''.
	Note for success rate, the higher
	the better; for query number, the smaller the better.}
	\scalebox{1.0}{
		\begin{tabular}{ccccccc}
			\toprule
			\multirow{2}[4]{*}{Alg.} & \multicolumn{2}{c}{IMDB} & \multicolumn{2}{c}{SST-2} & \multicolumn{2}{c}{SNLI} \\
			\cmidrule(lr){2-3} \cmidrule(lr){4-5} \cmidrule(lr){6-7}   & BiLSTM & BERT  & BiLSTM & BERT  & BiLSTM & BERT \\
			\midrule
			LS    & \textbf{99.93}\textbar2220 & \textbf{94.16}\textbar4332 & \textbf{94.07}\textbar295 & \textbf{88.90}\textbar343 & \textbf{74.71}\textbar246 & \textbf{76.43}\textbar233 \\
			
			Greedy & 99.83\textbar\textbf{2103} & 91.21\textbar\textbf{3967} & 93.56\textbar\textbf{257} & 88.65\textbar\textbf{266} & 74.48\textbar\textbf{173} & 75.78\textbar\textbf{174} \\
			\bottomrule
	\end{tabular}}
	\label{tab:success_rate_query_number_greedy}
\end{table*}

\section{Complete Evaluation Results of Adversarial Example Quality}
Table~\ref{tab:quality_GA}, Table~\ref{tab:quality_pwws} and Table~\ref{tab:quality_textfooler} list the results in terms of modification rate (Modi.), grammaticality (IPW) and fluency (PPL), in the scenarios introduced by Alzantot \textit{et al.}~\cite{AlzantotSEHSC18}, Ren \textit{et al.}~\cite{RenDHC19} and Jin \textit{et al.}~\cite{RenDHC19}, respectively.
It can be observed that in most scenarios the adversarial examples crafted by LS are better than that crafted by the baselines.
Note the quality of the adversarial examples crafted in the scenarios of Tan \textit{et al.}~\cite{TanJKS20} is not evaluated.
The reason is that it is aimed at revealing the discrimination of DNNs that are trained on perfect Standard English corpora, against minorities from nonstandard linguistic backgrounds.
Under its setting the adversarial examples crafted by Inflectional Morphology will naturally introduce quality issues, such as grammatical error and typos.
This is actually very different from the other three approaches which try to craft indistinguishable adversarial examples with correct grammaticality.

\section{Attack Performance of the Greedy Algorithm}
Table~\ref{tab:success_rate_query_number_greedy} presents the attack performance of LS and the Greedy algorithm in the scenarios of Zang \textit{et al.}~\cite{ZangQYLZLS20}.
Compared to Greedy, generally LS achieves higher success rates, with a few additional queries.





\ifCLASSOPTIONcaptionsoff
  \newpage
\fi

\end{document}